%% file: main.tex
\documentclass{article}
\usepackage{hyperref}
\usepackage{listings, color}
\usepackage{tabularx}

\usepackage{tabularx}
\usepackage{caption}
\usepackage{wrapfig}
\usepackage{adjustbox}
\usepackage{hyperref}
\hypersetup{
    colorlinks=true,
    linkcolor=blue,
    filecolor=magenta,      
    urlcolor=cyan,
}
\usepackage{url}            % simple URL typesetting
\usepackage{booktabs}       % professional-quality tables
\usepackage{amsfonts}       % blackboard math symbols
\usepackage{nicefrac}       % compact 
\usepackage{microtype}      %

\usepackage{framed}
\usepackage{amssymb}
\usepackage{amsfonts}
\usepackage{mathrsfs}
\usepackage{changepage}

\usepackage{mathtools}
\usepackage{array}
\usepackage{amsthm}
\usepackage{verbatim} 
\usepackage{enumerate}
\usepackage{bbm}
\usepackage{commath}
\usepackage{wrapfig}
\usepackage{amsbsy}
\usepackage{float}
\usepackage{amsmath}
\usepackage{algorithm}
\usepackage[numbers]{natbib}
\usepackage{tabularx}
\usepackage{listings}
\usepackage{enumitem}
\usepackage{lipsum}

\usepackage[final]{neurips_2021}

\newtheorem{manualtheoreminner}{Definition}
\newenvironment{manualdef}[1]{%
  \renewcommand\themanualtheoreminner{#1}%
  \manualtheoreminner
}{\endmanualtheoreminner}

\newtheorem{manualpropinner}{Proposition}
\newenvironment{manualprop}[1]{%
  \renewcommand\themanualtheoreminner{#1}%
  \manualpropinner
}{\endmanualtheoreminner}

\newcommand\blfootnote[1]{%
  \begingroup
  \renewcommand\thefootnote{}\footnote{#1}%
  \addtocounter{footnote}{-1}%
  \endgroup
}

\title{Can contrastive learning avoid shortcut solutions?}

\author{%
  Joshua Robinson \\
  MIT CSAIL \& LIDS\\
  \texttt{joshrob@mit.edu} 
  \And
  Li Sun  \\
  University of Pittsburgh\\
  \texttt{lis118@pitt.edu} \\
  \And
  Ke Yu  \\
  University of Pittsburgh\\
  \texttt{yu.ke@pitt.edu} \\
  \And
  Kayhan Batmanghelich  \\
  University of Pittsburgh\\
  \texttt{kayhan@pitt.edu} \\\And
  Stefanie Jegelka  \\
  MIT CSAIL\\
  \texttt{stefje@csail.mit.edu} \\
    \And
    Suvrit Sra \\
  MIT LIDS\\
  \texttt{suvrit@mit.edu} 
}

\begin{document}

\maketitle

\input{abstract_v1}

\input{intro}
\input{related}

\input{task-design}

\input{method}

\input{experiments}

\input{understanding}

\input{discussion}

{
  \bibliographystyle{plainnat}
  \bibliography{egbib}

}

\vfill
\pagebreak
\appendix

\input{appendix}

\end{document}

%% file: abstract_v1.tex
\begin{abstract} 
The generalization of representations learned via contrastive learning depends crucially on  what features of the data are extracted. However, we observe that the contrastive loss does not always sufficiently guide which features are extracted, a behavior that can negatively impact the performance on downstream tasks via ``shortcuts'', i.e., by inadvertently suppressing important predictive features. We find that feature extraction is influenced by  the \emph{difficulty} of the so-called instance discrimination task (i.e., the task of discriminating pairs of similar points from pairs of dissimilar ones).
Although harder pairs improve the representation of some features, the improvement comes at the cost of suppressing previously well represented features.
In response, we propose \emph{implicit feature modification} (IFM), a method for altering positive and negative samples in order to guide contrastive models towards capturing a wider variety of predictive features. Empirically, we observe that IFM reduces feature suppression, and as a result improves performance on vision and medical imaging tasks. The code is available at: \url{https://github.com/joshr17/IFM}.
\end{abstract}

%% file: intro.tex
\vspace{-15pt}
\section{Introduction}
\vspace{-5pt}
Representations trained with contrastive learning are adept at solving various vision tasks including classification, object detection, instance segmentation, and more \cite{chen2020simple,he2020momentum,tian2019contrastive}. In contrastive learning, encoders are trained to discriminate pairs of positive (similar) inputs from a selection of negative (dissimilar) pairs. This task is called \emph{instance discrimination}: It is often framed using the InfoNCE loss \cite{gutmann2010noise,oord2018representation}, whose minimization forces encoders to extract input features that are sufficient to discriminate similar and dissimilar pairs.\blfootnote{Correspondence to Joshua Robinson ({\texttt{joshrob@mit.edu}}).}

However, learning  features that are discriminative during training does not guarantee a model will generalize.
Many studies find  inductive biases in supervised learning toward \emph{simple} ``shortcut'' features and decision rules~\cite{hermann2020shapes,huh2021low,nguyen2020wide} which result in unpredictable model behavior under perturbations \cite{ilyas2019adversarial,szegedy2013intriguing} and failure outside the training distribution \cite{beery2018recognition,recht2019imagenet}.  Simplicity bias has various potential sources \cite{geirhos2020shortcut} including training methods \cite{chizat2020implicit,lyu2019gradient, soudry2018implicit} and  architecture design \cite{geirhos2018imagenet,hermann2019origins}. Bias towards shortcut decision rules also hampers transferability in contrastive learning  \cite{chen2020intriguing}, where it is in addition influenced by the instance discrimination task. These difficulties lead us to ask: can the contrastive instance discrimination task itself be modified to avoid learning shortcut solutions? 

\iffalse
there is an additional factor influencing contrastive feature learning: unlike with supervised learning, contrastive instance discrimination tasks themselves must be \emph{designed}. The relation between contrastive instance discrimination tasks and the process leading to representation of some input features but not others is not yet clear \cite{chen2020intriguing}. 
Better understanding of the phenomenon is crucial to understanding the usefulness of a representation -- if an encoder represents one feature at the expense of another, this determines which tasks the encoder can solve, which it fails, and its potential for out-of-distribution generalization. 

So instead, in this paper we ask: is there a way to modify contrastive learning to better capture the variety of predictive features? Unlike supervised learning, contrastive learning introduces an additional source of feature learning bias: contrastive instance discrimination tasks -- i.e. positive and negative pairs -- must themselves be \emph{designed}.
\fi

We approach this question by studying the relation between contrastive instance discrimination and feature learning. First, we theoretically explain why optimizing the InfoNCE loss alone does not guarantee avoidance of shortcut solutions that \emph{suppress} (i.e., discard) certain input features   \cite{chen2020intriguing,geirhos2020shortcut}. Second, despite this negative result, we show that it is still possible to trade off representation of one feature for another using simple  methods for adjusting the difficulty of instance discrimination. 
However, these methods have an important drawback: improved learning of one feature often comes at the cost of harming  another. That is, feature suppression is still prevalent. In response, we propose \emph{implicit feature modification}, a technique that encourages encoders to discriminate instances using multiple input features. Our method introduces no computational overhead, reduces feature suppression (without trade-offs), and improves generalization on  various downstream tasks.

\textbf{Contributions.} In summary, this paper makes the following main contributions:
\vspace{-5pt}
\begin{enumerate}
\setlength{\itemsep}{1pt}
    \item It analyzes feature suppression in contrastive learning, and explains why feature suppression can occur when optimizing the InfoNCE loss. 
    \item It studies the relation between instance discrimination tasks and feature learning; concretely, adjustments to instance discrimination difficulty leads to different features being learned.
    \item It proposes \emph{implicit feature modification}, a simple and efficient method that reduces the tendency to use feature suppressing shortcut solutions and improves generalization. 
\end{enumerate}

%% file: related.tex
\vspace{-10pt}
\subsection{Related work}
\vspace{-5pt}

Unsupervised representation learning is enjoying a renaissance  driven by steady advances in effective frameworks  \cite{caron2020unsupervised,chen2020simple,he2020momentum,hjelm2018learning,oord2018representation,tian2019contrastive,tian2020makes,zbontar2021barlow}. As well as many effective  contrastive methods, Siamese approaches that avoid representation collapse without explicitly use of negatives have also been proposed \cite{chen2020exploring,grill2020bootstrap,zbontar2021barlow}. Pretext task design has been at the core of progress in self-supervised learning. Previously popular tasks include  image colorization \cite{zhang2016colorful} and inpainting \cite{pathak2016context}, and theoretical work shows pre-trained encoders can provably generalize if a pretext task necessitates the learning of features that solve downstream tasks \cite{lee2020predicting,robinson2020strength}. In contrastive learning, augmentation strategies are a key design component  \cite{chen2020simple,wang2020unsupervised,wang2021augmentations}, as are negative mining techniques  \cite{chuang2020debiased,he2020momentum,kalantidis2020hard,robinson2020contrastive}. While feature learning in contrastive learning has received less attention, recent work finds that low- and mid-level features are more important for transfer learning \cite{zhao2020makes}, and feature suppression can occur \cite{chen2020intriguing} just as with supervised learning \cite{geirhos2018imagenet,hermann2020shapes}. Combining contrastive learning with an auto-encoder has also been considered \cite{li2020information}, but was found to harm representation of some features in order to avoid suppression of others. Our work is distinguished from prior work through our focus on how the design of the instance discrimination task itself affects which features are learned.

%% file: task-design.tex
\vspace{-5pt}
\section{Feature suppression in contrastive learning}\label{sec: task design}
\vspace{-5pt}

 \begin{figure}[t] %{6.5cm}
 \vspace{-5pt}
  \includegraphics[width=\textwidth]{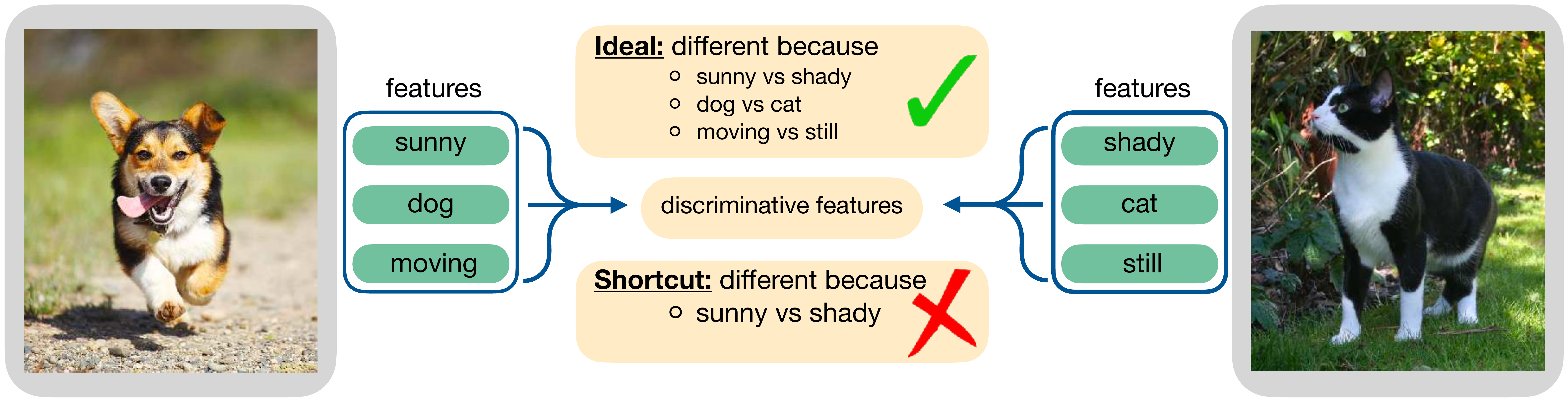}
  \vspace{-15pt}
     \caption{An ideal encoder would  discriminate between instances using multiple distinguishing features instead of finding simple shortcuts that suppress features. We show that InfoNCE-trained encoders can suppress features (Sec. \ref{sec: optimizing the InfoNCE loss can lead to feature suppression}). However, making instance discrimination harder during training can trade off representation of different features (Sec. \ref{sec: Controlling feature suppression}). To avoid the need for trade-offs we propose \emph{implicit feature modification} (Sec. \ref{sec: method}), which reduces suppression in general, and improves generalization (Sec. \ref{sec: experiments}). }
      \label{fig: fig1}
        \vspace{-15pt}
\end{figure}

Feature suppression refers to the phenomenon where, in the presence of multiple predictive input features, a model uses only a subset of them and ignores the others. The selected subset often corresponds to intuitively ``simpler'' features, e.g., color as opposed to shape. Such features lead to ``shortcut'' decision rules that might perform well on training data, but can harm generalization and lead to poor robustness to data shifts. Feature suppression has been identified as a common problem in deep learning \cite{geirhos2020shortcut}, and both  supervised and contrastive learning suffer from biases induced by the choice of optimizer and architecture.
However, contrastive learning bears an additional potential source of bias: \emph{the choice of instance discrimination task}. Which positive and negative pairs are presented  critically affects which features are discriminative, and hence which features are learned. In this work we study the relation between feature suppression and instance discrimination.

First, we  explain why optimizing the InfoNCE loss is insufficient in general to avoid feature suppression, and show how it can lead to counter-intuitive generalization (Sec. \ref{sec: optimizing the InfoNCE loss can lead to feature suppression}). Given this negative result, we then ask if it is at least possible to \emph{control} which features a contrastive encoder learns? We find that this is indeed the case, and that adjustments to the instance discrimination task lead to different features being learned (Sec. \ref{sec: Controlling feature suppression}). However, the primary drawback of these adjustments is that improving one feature often comes at the cost of harming representation of another. That is, feature suppression is still prevalent. Addressing this drawback is the focus of Sec. \ref{sec: method}. 
\vspace{-5pt}
\subsection{Setup and definition of feature suppression}\label{sec: notarion setup}
\vspace{-5pt}
Formally, we assume that the data has 
 underlying feature spaces $\mathcal Z^1, \ldots , \mathcal Z^n$ with a distribution $p_j$ on each $\mathcal Z^j$. Each $j \in [n]$, corresponding to a latent space $\mathcal Z^j$, models a distinct feature. We write the product as $\mathcal Z^S = \prod_{j\in S} \mathcal Z^j$, and simply write $\mathcal Z$ instead of $\mathcal Z^{[n]}$ where $[n] = \{1, \ldots , n\}$. A set of features $z  = (z^j)_{j \in [n]}\in \cal Z$ is generated by sampling each coordinate $z^j \in \mathcal Z^j$ independently, and we denote the measure on $\mathcal Z$ induced by $z$ by $\lambda$. Further, let $\lambda(\cdot | z^S)$ denote the conditional measure on  $\mathcal Z$ for fixed $z^S$. For $S \subseteq [n]$ we use $z^S$ to denote the projection of $z$ onto $ \mathcal Z^S$. Finally, an injective map $g: \mathcal Z \rightarrow \mathcal X$ produces observations $x=g(z)$.

Our aim is to train an encoder $f :\mathcal X \rightarrow  \mathbb{S}^{d-1}$ to map input data $x$ to the surface of the unit sphere $\mathbb{S}^{d-1}= \{ u \in \mathbb{R}^d : \| u\|_2 = 1\}$ in such a way that $f$ extracts useful information. To formally define feature suppression, we need the \emph{pushforward} $h \# \nu(V)=\nu(h^{-1}(V))$ of a  measure $\nu$ on a space $\cal U$ for a measurable map $h : \cal U \rightarrow \cal V$ and  measurable $V \subseteq \cal V$, where $h^{-1}(V)$ denotes the preimage.

%
%Finally, for a measure $\nu$ on a space $\cal U$ and a measurable map $h : \cal U \rightarrow \cal V$,  let $h \# \nu$ denote the pushforward of $\nu$ by $h$ defined by the formula $h \# \nu(V)=\nu(h^{-1}(V))$ for measurable $V \subseteq \cal V$, where $h^{-1}(V)$ denotes the preimage. Each $j \in [n]$, corresponding to a latent space $\mathcal Z^j$ models a distinct feature, and the upcoming definition precisely states what it means for an encoder $f$ to suppress or distinguish a set of features $S \subseteq [n]$.
%
%
\begin{defn}
Consider an encoder $f :\mathcal X \rightarrow  \mathbb{S}^{d-1}$ and features $S \subseteq [n]$. For each $z^S \in %\in
\mathcal Z^S$, let $\mu (\cdot| z^S) = (f\circ g) \# \lambda (\cdot| z^S) $ be the pushforward measure  on $\mathbb{S}^{d-1}$ by $f\circ g$ of the conditional $\lambda (\cdot| z^S)$. 
\begin{enumerate}
    \vspace*{-7pt}   
    \setlength{\itemsep}{0pt}
    \item  $f$ \emph{suppresses} $S$ if for any pair $z^S , \bar{z}^S\in \mathcal Z^S$, we have $\mu (\cdot| z^S) =  \mu (\cdot| {\bar{z}^S})$. 
    %\vspace{-3pt}
    \item  $f$ \emph{distinguishes} $S$ if for any pair of distinct $z^S, \bar{z}^S\in \mathcal Z^S$, measures $\mu (\cdot| z^S), \mu (\cdot| {\bar{z}^S})$ have disjoint support. 
    \vspace{-5pt}
\end{enumerate}
\end{defn}
\vspace{-4pt}

Feature suppression is thus captured in a distributional manner, stating that $S$ is suppressed if the encoder distributes inputs in a way that is invariant to the value $z^S$. Distinguishing features, meanwhile, asks that the encoder $f$ separates points with different features $z^S$ into disjoint regions. We consider training an encoder $f :\mathcal X \rightarrow  \mathbb{S}^{d-1} $ to optimize the InfoNCE loss \cite{oord2018representation,gutmann2010noise},
%\vspace{-12pt}
%
\begin{equation}\label{eqn: InfoNCE loss}
\mathcal L_m(f) = \mathbb{E}_{x,x^+,\{x_i^-\}_{i=1}^m} \bigg [ - \log \frac{e^{f(x)^\top f(x^+) /\tau}}{e^{f(x)^\top f(x^+) / \tau} + \sum_{i=1}^m e^{f(x)^\top f(x_i^-) / \tau}}\bigg ],  
\end{equation} 
%\vspace{-10pt}
%
where $\tau$ is known as the \emph{temperature}. Positive pairs $x,x^+$ are generated by first sampling $z \sim \lambda$, then independently sampling two random augmentations $a,a^+ \sim \mathcal A$, $a :\mathcal X \rightarrow \mathcal X$ from a distribution $\mathcal A$,
and setting $x=a(g(z))$ and $x^+=a^+(g(z))$. We assume $\mathcal A$ samples the identity function $a(x)=x$ with non-zero probability (``$x$ is similar to itself''), and that there are no collisions: $a(x)\neq a'(x')$ for all $a,a'$, and all $x\neq x'$. Each negative example $x_i^-$ is generated as $x^-_i = a_i (g (z_i))$, by independently sampling features $z_i \sim \lambda$ and an augmentation $a_i \sim \mathcal A$.

\vspace{-5pt}
\subsection{Why optimizing the InfoNCE loss can still lead to feature suppression}\label{sec: optimizing the InfoNCE loss can lead to feature suppression}
\vspace{-5pt}

Do optimal solutions to the InfoNCE loss automatically avoid shortcut solutions? Unfortunately, as we show in this section, this is not the case in general; there exist both optimal solutions of the InfoNCE loss that do and solutions that do not suppress a given feature. Following previous work \cite{robinson2020contrastive,wang2020understanding,zimmermann2021contrastive}, we analyze the loss as the number of negatives goes to infinity,
\begin{equation*}
\mathcal L = \lim_{m \rightarrow \infty } \big \{ \mathcal L_m(f) - \log m  - \tfrac2{\tau}\big \} = \tfrac{1}{2\tau}\mathbb{E}_{x,x^+}  \| f(x) - f(x^+) \|^2 + \mathbb{E}_{x^+ } \log \big [  \mathbb{E}_{x^-} e^{f(x^+)^\top f(x^-)/\tau} \big ].
\end{equation*}
We subtract $\log m$ to ensure the limit is finite, and use $x^-$ to denote a random sample with the same distribution as $x^-_i$. 
Prop.~\ref{thm: new main result} (proved in App.~ \ref{appx: proofs}) shows that, assuming the marginals $p_j$ are uniform, the InfoNCE loss is optimized both by encoders that suppress feature $j$, and by encoders that distinguish $j$. 

\begin{wrapfigure}{r}{0.54\textwidth}
\vspace{-5pt}
  \begin{center}
    \includegraphics[width=0.30\textwidth]{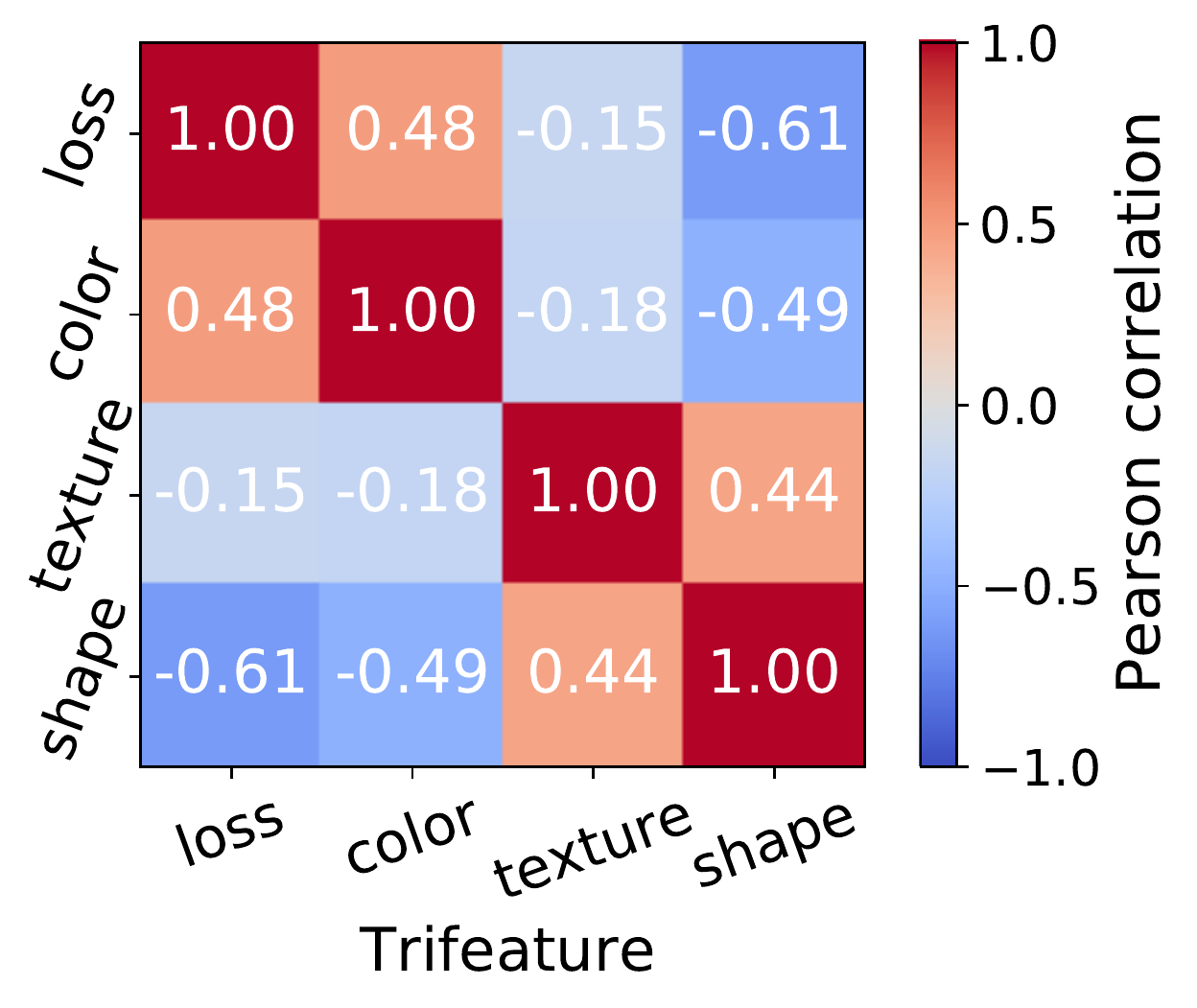}
    \includegraphics[width=0.23\textwidth]{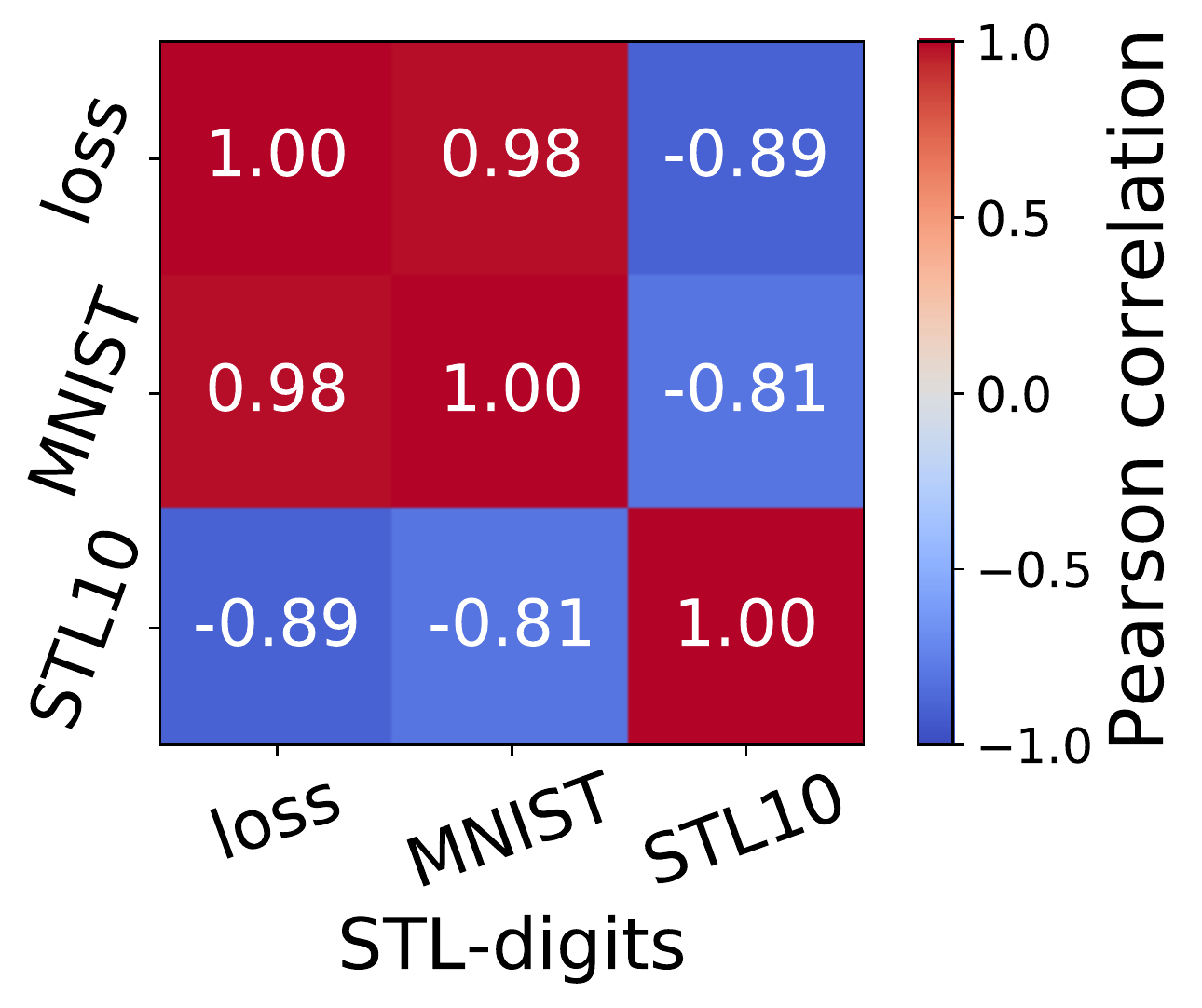}
    \vspace{-20pt}
    \caption{Linear readout error on different downstream tasks can be negatively correlated. Further, lower InfoNCE loss does not always yield not lower error: error rates on texture, shape and STL10 prediction are \emph{negatively correlated} with InfoNCE loss. }
    \label{fig: trifeature corr}
  \end{center}
  \vspace{-20pt}
\end{wrapfigure}

\begin{prop}\label{thm: new main result} Suppose that $p_j$ is uniform on $\mathcal Z^j = \mathbb{S}^{d-1}$ for all $j \in [n]$. Then for any feature $j \in [n]$ there exists an encoder $f_\text{supp}$ that suppresses feature $j$ and encoder $f_\text{disc}$ that discriminates $j$ but both attain $\min_{f: \text{ measurable}}\mathcal L(f)$.
\end{prop}
\vspace{-5pt}

The condition that $p_j$ is uniformly distributed on $\mathcal Z^j = \mathbb{S}^{d-1}$ is similar to conditions used in previous work \cite{zimmermann2021contrastive}. Prop.~\ref{thm: new main result} shows that empirical observations of feature suppression \cite{chen2020intriguing} (see also Fig.~\ref{fig: feature suppression}) are not simply due to a failure to sufficiently optimize the loss, but that the possibility of feature suppression is \emph{built into} the loss. What does Prop.~\ref{thm: new main result} imply for the generalization behavior of encoders? Besides explaining why feature suppression can occur, Prop.~\ref{thm: new main result} also suggests another counter-intuitive possibility: \emph{lower InfoNCE loss may actually lead to worse performance on some tasks}. 

To empirically study whether this possibility manifests in practice, we use two datasets with known semantic features: (1) In the Trifeature data, \cite{hermann2020shapes} each image is $128 \times 128$ and has three features: color, shape, and texture, each taking possible 10 values. See Fig. \ref{fig: trifeature samples}, App.~\ref{appdx: experiments} for sample images.
(2) In  the STL-digits data, samples combine MNIST digits and STL10 objects by placing copies of a randomly selected MNIST digit on top of an STL10 image. See Fig.~\ref{fig: stl-digits samples} App.~\ref{appdx: experiments} for sample images.

We train encoders with ResNet-18 backbone using SimCLR \cite{chen2020simple}. To study correlations between the loss value and error on downstream tasks, we train $33$ encoders on Trifeature and $7$ encoders on STL-digits with different hyperparameter settings (see App.~\ref{appdx: stl-digits suppression} for full details on training and hyperparameters). For Trifeature, we compute the Pearson correlation between InfoNCE loss and linear readout error when  predicting $\{\text{color, shape, texture}\}$. Likewise, for STL-digits we compute correlations between the InfoNCE loss and MNIST and STL10 prediction error.

Fig. \ref{fig: trifeature corr} shows that performance on different downstream tasks is not always positively correlated. For Trifeature, color error is negatively correlated with shape and texture,  while for STL-digits there is a strong negative correlation between MNIST digit error  and STL10 error. Importantly, lower InfoNCE loss is correlated with lower prediction error for color and MNIST-digit, but with \emph{larger} error for shape, texture and STL10. Hence, lower InfoNCE loss can improve representation of some features (color, MNIST digit), but may actually \emph{hurt} others.
%(texture, shape, STL10 object). 
This conflict is likely due to the simpler color and MNIST digit features being used as shortcuts. Our observation is an important addition to the statement of \citet{wang2020understanding} that lower InfoNCE loss improves generalization: the situation is more subtle -- whether lower InfoNCE helps generalization on a task depends on the use of shortcuts.

%\vspace{-10pt}
\subsection{Controlling feature learning via the difficulty of instance discrimination}\label{sec: Controlling feature suppression}
\vspace{-5pt}

The previous section showed that the InfoNCE objective has solutions that suppress features. Next, we ask what factors determine which features are suppressed? Is there a way to target \emph{specific} features and ensure they are encoded? One idea is to use \emph{harder} positive and negative examples. Hard examples are precisely those that are not easily distinguishable using the currently extracted features. So, a focus on hard examples may change the scope of the captured features. To test this hypothesis, we consider two methods for adjusting the difficulty of positive and negative samples: 
\vspace{-3pt}
\begin{enumerate}\setlength{\itemsep}{0pt}
    \item Temperature $\tau$ in the InfoNCE loss (Eqn. \ref{eqn: InfoNCE loss}). Smaller $\tau$ places higher importance on positive an negative pairs with high similarity \cite{wang2020understandingbehaviour}.
    \item Hard negative sampling method of Robinson et al. \cite{robinson2020contrastive}, which uses importance sampling to sample harder negatives. The method introduces a hardness concentration parameter $\beta$, with larger $\beta$ corresponding to harder negatives (see \cite{robinson2020contrastive} for full details). 
\end{enumerate}
\vspace{-5pt}

Results reported in Fig.~\ref{fig: feature suppression} (also Fig.~\ref{fig: stl10 appdx} in App.~\ref{appdx: stl-digits suppression}) show that varying instance discrimination difficulty---i.e., varying temperature $\tau$ or hardness concentration $\beta$---enables  trade-offs between which features are represented. On Trifeature, easier instance discrimination (large $\tau$, small $\beta$) yields good performance on `color'---an ``easy'' feature for which a randomly initialized encoder already has high linear readout accuracy---while generalization on the harder texture and shape features is poor. The situation \emph{reverses} for harder instance discrimination (small $\tau$, large $\beta$). We hypothesize that the use of ``easy'' features with easy instance discrimination is analogous to simplicity biases in supervised deep networks  \cite{hermann2019origins,huh2021low}. As with supervised learning \cite{geirhos2018imagenet,hermann2019origins}, we observe a bias for texture over shape in convolutional networks, with texture prediction always outperforming shape. 

That there are simple levers for controlling which features are learned already distinguishes contrastive learning from supervised learning, where attaining such control is less easy (though efforts have been made \cite{jacobsen2018excessive}). However, 
%these results only show it is possible to sacrifice representation 
these results show  that representation of one feature must be sacrificed in exchange for learning another one better. To understand how to develop methods for improving feature representation without suppressing others, the next result (proof in  App.~\ref{appx: proofs}) examines more closely \emph{why} there is a relationship between (hard) instance discrimination tasks and feature learning.

\begin{prop}[Informal] Suppose that $p_j$ is uniform on $\mathcal Z^j = \mathbb{S}^{d-1}$ for all $j \in [n]$.  Further, for $S \subseteq [n]$ suppose that $x,x^+,\{x_i^-\}_i$ are conditioned on the event that they have the same features $S$. Then any $f$ that minimizes the (limiting) InfoNCE loss suppresses features $S$.
\label{prop: S held constant implies suppressed}
%\vspace{-10pt}
\end{prop}

 \begin{figure}[t] %{6.5cm}
  \includegraphics[width=\textwidth]{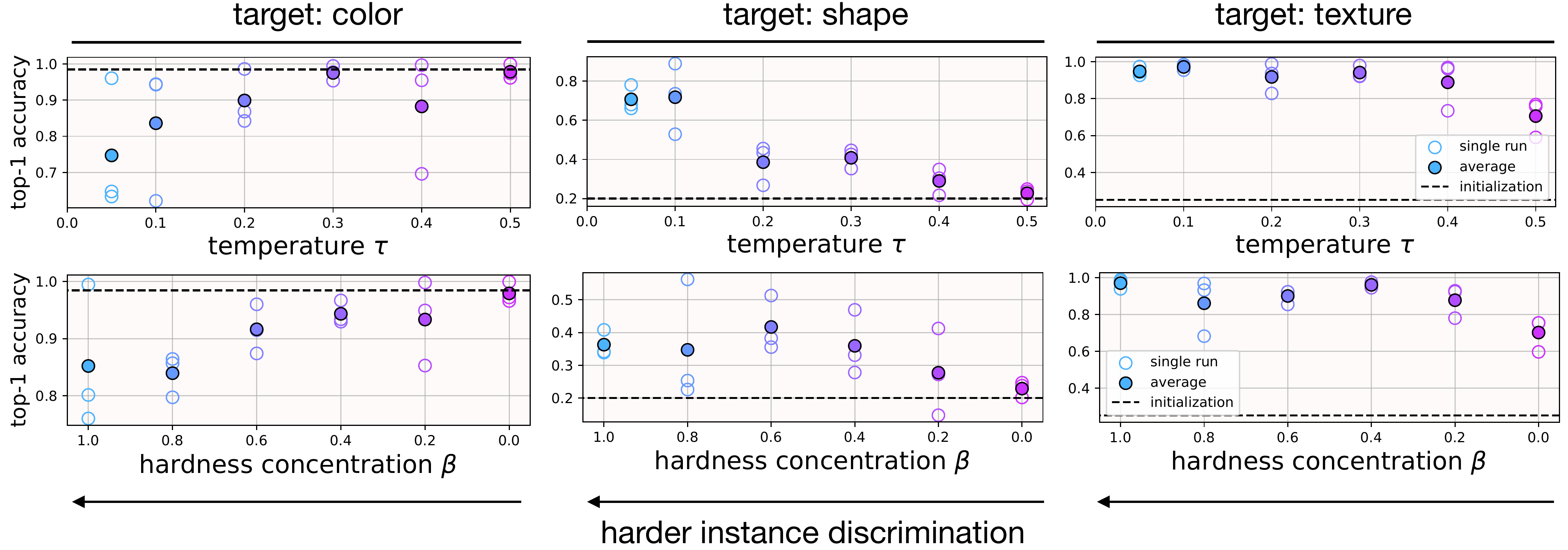}
  \vspace{-14pt}
\caption{Trifeature dataset \cite{hermann2020shapes}. 
The \emph{difficulty} of instance discrimination affects which features are learned (Sec.~\ref{sec: Controlling feature suppression}). When instance discrimination is easy (big $\tau$, small $\beta$), encoders represent color well and other features badly. When instance discrimination is hard (small $\tau$, big $\beta$), encoders represent more challenging shape and texture features well, at the expense of color. }
\vspace{-10pt}
\label{fig: feature suppression}
\end{figure}

The positive and negative instances in Prop.~\ref{prop: S held constant implies suppressed} must be distinguished with features in $S^c$.  Relating this point to the above observations, assume that an encoder exclusively uses features $S$. Any positives and negatives that do not (much) differ in features $S$ are difficult for the encoder. By Prop.~\ref{prop: S held constant implies suppressed}, focusing the training on these difficult examples pushes the encoder to instead use features in $S^c$, i.e., to learn \emph{new} features. But at the same time, the proposition also says that a strong focus on such hard negative pairs leads to suppressing the originally used features $S$, explaining the results in Fig.~\ref{fig: feature suppression}. 
While the two techniques for adjusting instance difficulty we studied were unable to avoid feature suppression, this insight forms the motivation for \emph{implicit feature modification}, which we introduce next.

%% file: method.tex
\vspace{-5pt}
\section{Implicit feature modification for reducing feature suppression}\label{sec: method}
\vspace{-5pt}

The previous section found that simple adjustments to instance discrimination \emph{difficulty} could significantly alter which features a model learns.
Prop. \ref{prop: S held constant implies suppressed} suggests that this ability to modify which features are learned stems from holding features constant across positive and negative samples. However, these methods were unable to avoid trade-offs in feature representation (Fig. \ref{fig: feature suppression}) since features that are held constant are themselves suppressed (Prop. \ref{prop: S held constant implies suppressed}). 

To avoid this effect, we develop a technique that  \emph{adaptively} modifies samples to remove whichever features are used to discriminate a particular positive pair from negatives, then trains an encoder to discriminate instances using \emph{both} the original features, and the features left over after modification. While a natural method for modifying features is to directly transform raw input data, it is very challenging to modify the semantics of an input in this way. So instead we propose modifying features by applying transformations to  encoded samples $v=f(x)$. Since we modify the encoded samples, instead of raw inputs $x$, we describe our method as \emph{implicit}.

We set up our notation. Given batch $x,x^+,\{x^-_i\}_{i=1}^m$ we write $v=f(x)$, $v^+=f(x^+)$, and $v_i^-=f(x_i^-)$ to denote the corresponding embeddings. As in Eqn. \ref{eqn: InfoNCE loss}, the point-wise InfoNCE loss is,
\begin{equation*}
    \ell(v,v^+, \{ v^-_i\}_{i=1}^m) = - \log \frac{e^{v^\top v^+/ \tau}}{e^{v^\top v^+ / \tau} + \sum_{i=1}^m e^{v^\top v^-_i/ \tau}}.
\end{equation*}
\begin{defn}[Implicit feature modification]\label{principle: motivating adv framework} Given budget $\bm{\varepsilon} \in \mathbb{R}^m_+$, and encoder $f : \mathcal X \rightarrow \mathbb{S}^d$, an adversary removes features from $f$ that discriminates batch $x,x^+,\{x^-_i\}_{i=1}^m$ by maximizing the point-wise InfoNCE loss,
$
 \ell_{\bm{\varepsilon}}(v,v^+, \{ v^-_i \}_{i=1}^m)= \max_{\delta^+ \in \mathcal B_{\varepsilon^+}, \{\delta_i^- \in \mathcal B_{\varepsilon_i} \}_{i=1}^m} \ell(v,v^+ + \delta^+, \{ v^-_i + \delta^-_i\}_{i=1}^m)$.

\end{defn}
\vspace{-5pt}

Here $\mathcal B_\varepsilon$ denotes the $\ell_2$-ball of radius $\varepsilon$.  Implicit feature modification (IFM) removes components of the current representations that are used to discriminate positive and negative pairs. In other words, the embeddings of positive and negative samples are modified to remove well represented features. So, if the encoder is currently using a simple shortcut solution, IFM removes the features used, thereby encouraging the encoder to also discriminate instances using other features.  By applying perturbations in the embedding space IFM can modify high level semantic features (see Fig. \ref{fig: visualization}), which is extremely challenging when applying perturbations in input space.  In order to learn new features using the perturbed loss while still learning potentially complementary information using the original InfoNCE objective, we propose optimizing the the multi-task objective $\min_f \{ \mathcal L (f) + \alpha \mathcal L_{\bm{\varepsilon}} (f)\}/2$ where $\mathcal L_{\bm{\varepsilon}} = \mathbb{E}\ell_{\bm{\varepsilon}} $ is the adversarial perturbed loss, and  $\mathcal L$  the standard InfoNCE loss. For simplicity, all experiments set the balancing parameter $\alpha=1$ unless explicitly noted, and all take $\varepsilon^+, \varepsilon^-_i$ to be equal, and denote this single value by $\varepsilon$. Crucially, $\ell_{\bm{\varepsilon}}$ can be  computed analytically and efficiently. 
\begin{lemma}
For any $v,v^+, \{ v^-_i \}_{i=1}^m \in \mathbb{R}^d$ we have,
\begin{equation*}
\nabla_{v^-_j} \ell  = \frac{e^{v^\top v^-_j / \tau}}{e^{v^\top v^+/ \tau} + \sum_{i=1}^m e^{v^\top v^-_i/ \tau}} \cdot \frac{v}{\tau}  \quad  \text{and} \quad  \nabla_{v^+} \ell  =  \bigg (   \frac{e^{v^\top v^+/ \tau}}{e^{v^\top v^+/ \tau} + \sum_{i=1}^m e^{v^\top v^-_i/ \tau}} -1 \bigg ) \cdot \frac{v}{\tau} .
\end{equation*}
In particular, $\nabla_{v^-_j} \ell \propto v $ and $\nabla_{v^+} \ell \propto -v $.
\end{lemma}
\vspace{-5pt}
 This expression shows that the adversary perturbs $v^-_j$ (resp. $v^+$) in the direction of the anchor $v$ (resp $-v$). Since the derivative directions are \emph{independent} of $\{v_i^-\}_{i=1}^m$ and $v^+$, we can analytically compute optimal perturbations in $\mathcal B_\varepsilon$. Indeed, following the constant ascent direction shows the optimal updates are simply $v^-_i \leftarrow v^-_i + \varepsilon_i   v  $ and $v^+ \leftarrow v^+ - \varepsilon^+   v $.  The positive (resp. negative) perturbations increase (resp. decrease) cosine similarity to the anchor $\text{sim}(v, v^-_i + \varepsilon_i   v) \rightarrow 1$ as $\varepsilon_i \rightarrow \infty$ (resp. $\text{sim}(v, v^+ - \varepsilon^+  v) \rightarrow -1$ as $\varepsilon^+ \rightarrow \infty$). In Fig. \ref{fig: visualization} we visualize the newly synthesized $v^-_i,v^+$ and find meaningful interpolation of semantics. Plugging the update rules for $v^+$ and $v_i^-$ into the point-wise InfoNCE loss yields,
\begin{equation}
 \ell_{\bm{\varepsilon}}(v,v^+, \{ v^-_i\}_{i=1}^m) = - \log \frac{e^{(v^\top v^+ - \varepsilon^+)/\tau}}{e^{(v^\top v^+ - \varepsilon^+)/\tau} + \sum_{i=1}^m e^{(v^\top v^-_i + \varepsilon_i)/\tau}}.
\label{eqn: attack both loss}
\end{equation}
In other words, IFM amounts to simply perturbing the logits -- reduce the positive logit by $\varepsilon^+/\tau$ and increase negative logits by $\varepsilon_i/\tau$. From this we see that $\ell_{\bm{\varepsilon}}$ is automatically symmetrized in the positive samples: perturbing $v$ instead of $v^+$ results in the exact same objective.  Eqn. \ref{eqn: attack both loss} shows that IFM re-weights each negative sample  by a factor  $e^{\varepsilon_i/\tau}$ and positive samples by $e^{-\varepsilon^+/\tau}$. 

\vspace{-5pt}
\subsection{Visualizing implicit feature modification}\label{sec: vis}
\vspace{-5pt}

With implicit feature modification, newly synthesized data points do not directly correspond to any ``true'' input data point.
However it is still possible to visualize the effects of implicit feature modification. To do this, assume access to a memory bank of input data $\mathcal M=\{x_i\}_i$. A newly synthesized sample $s$ can be approximately visualized by retrieving the 1-nearest neighbour using cosine similarity $\arg\min_{x \in \mathcal M} \text{sim}( s,  f(x))$ and viewing the image $x$ as an approximation to $s$. 

Fig. \ref{fig: visualization} shows results using a  ResNet-50 encoder trained using MoCo-v2 on ImageNet1K using the training set as the memory bank. For positive pair $v,v^+$ increasing $\varepsilon$ causes the semantics of $v$ and $v^+$ to diverge. For $\varepsilon=0.1$ a different car with similar pose and color is generated, for $\varepsilon=0.2$ the pose and color then changes, and finally for $\varepsilon=1$ the pose, color and type of vehicle changes. For negative pair $v,v^-$ the reverse occurs.  For $\varepsilon=0.1$, $v^-$ is a vehicle with similar characteristics (number of windows, color etc.), and with $\varepsilon=0.2$, the pose of the vehicle $v^+$ aligns with $v$. Finally for $\varepsilon=1$ the pose and color of the perturbed negative sample become aligned to the anchor $v$. In summary, implicit feature modification successfully \emph{modifies the feature content in positive and negative samples}, thereby altering which features can be used to  discriminate instances. 

\vspace{-5pt}

\paragraph{Related Work.} Several  works consider adversarial contrastive learning  \cite{ho2020contrastive,jiang2020robust,kim2020adversarial}  using PGD (e.g. FGSM) attacks to alter samples in input space. Unlike our approach, PGD-based attacks require costly inner-loop optimization.  Other work takes an adversarial viewpoint in input space for other self-supervised tasks e.g., rotations and jigsaws but uses an image-to-image network to simulate FGSM/PGD attacks \cite{minderer2020automatic}, introducing comparable computation overheads. They note that low-level (i.e., pixel-level) shortcuts can be avoided using their method. All of these works differ from ours by applying attacks in input space, thereby focusing on lower-level features, whereas ours aims to modify high-level features.  Fig. \ref{fig: simadv acl comparison} compares IFM to this family of input-space adversarial methods by  comparing to a top performing method ACL(DS) \cite{jiang2020robust}. We find that ACL improves robust accuracy under $\ell_\infty$-attack on input space (see \cite{jiang2020robust} for protocol details), whereas IFM improves standard accuracy (full details and discussion in Appdx. \ref{appendix: acl comparison}).  Synthesizing harder negatives in latent space using Mixup \cite{zhang2017mixup} has also been considered \cite{kalantidis2020hard} but does not take an adversarial perspective. Other work, AdCo \cite{hu2020adco},  also takes an adversarial viewpoint in latent space. There are several differences to our approach. 
AdCo perturbs all negatives using the same weighted combination of all the queries, whereas IFM perturbations are query specific.
In other words, IFM makes instance discrimination harder point-wise, whereas AdCo perturbation makes the InfoNCE loss larger \emph{on average} (see  Fig. \ref{fig: visualization} for visualizations of instance dependent perturbation using IFM). AdCo also treats the negatives as learnable parameters, introducing $\sim 1M$ more parameters and $\sim 7\%$ computational overhead, while IFM has no computational overhead and is implemented with only two lines of code (see Tab. \ref{tab: imagenet100} for empirical comparison). Finally, no previous work makes the connection between suppression of semantic features and adversarial methods in contrastive learning (see Fig. \ref{fig: simadv suppression}).

%% file: experiments.tex
\vspace{-10pt}
\section{Experimental results}\label{sec: experiments}
\vspace{-10pt}

 \begin{figure}[t] %{6.5cm}
 \begin{minipage}[t]{0.67\textwidth}
  \includegraphics[width=0.96\textwidth]{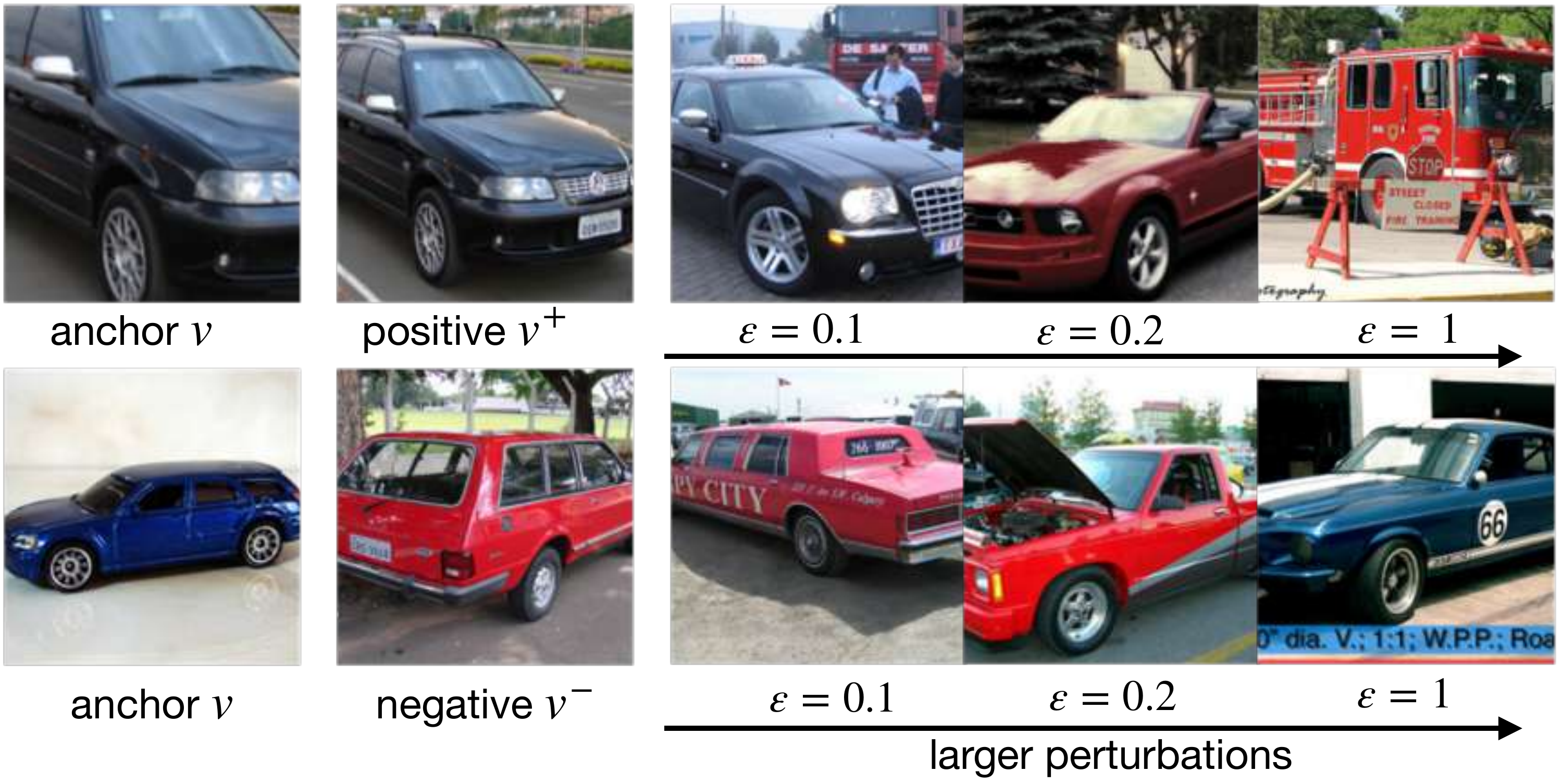}
   \vspace{-6pt}
  \caption{Visualizing implicit feature modification. \textbf{Top row:} progressively moving positive sample away from anchor. \textbf{Bottom row:} progressively moving negative sample away from anchor. In both cases, semantics such as color, orientation, and vehicle type are modified, showing the suitability of implicit feature modification for altering instance discrimination tasks.}
  \label{fig: visualization}
  \end{minipage}\hfill
    \begin{minipage}[t]{0.32\textwidth}
\includegraphics[width=0.96\textwidth]{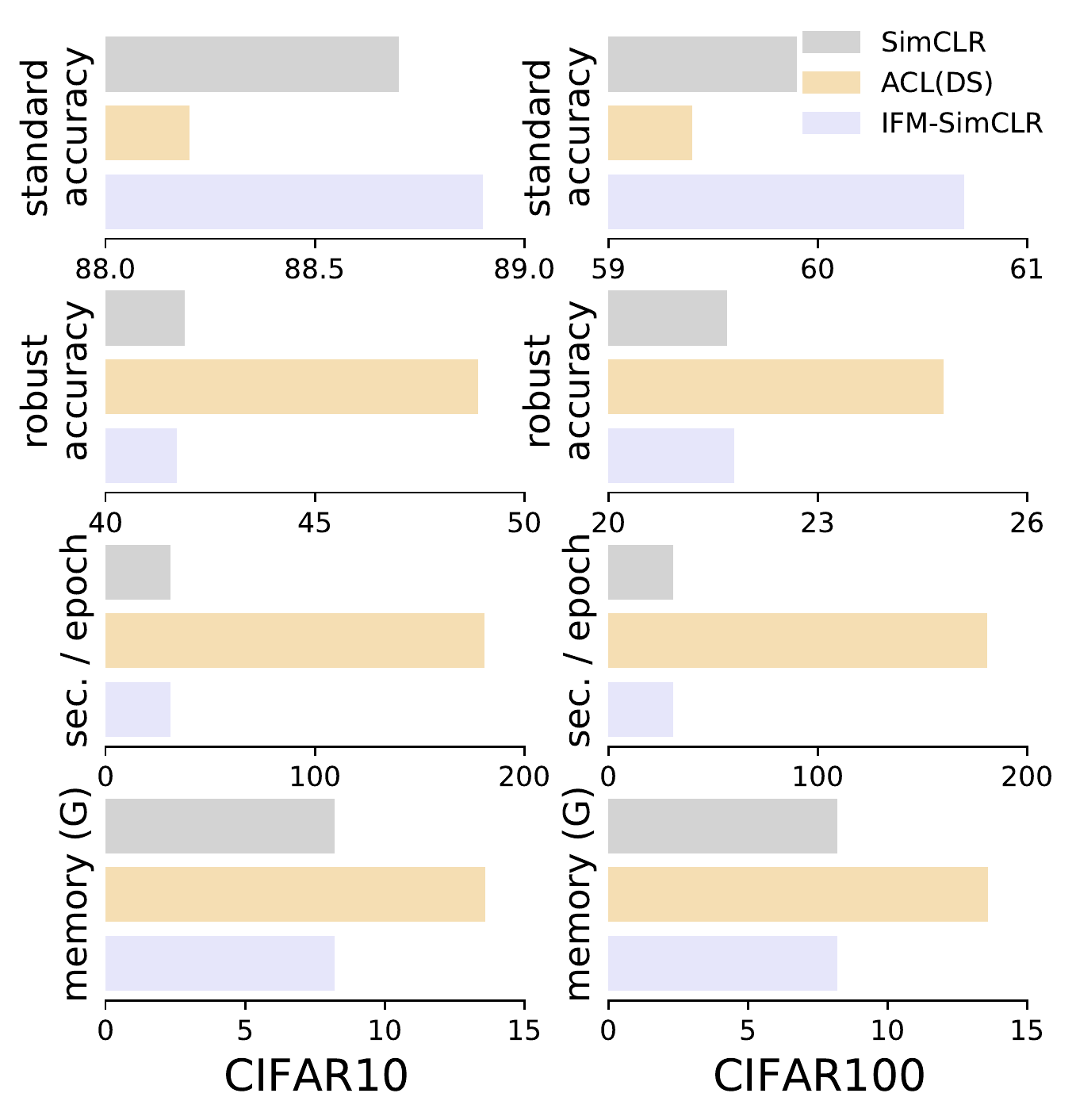}
\begin{adjustwidth}{7pt}{7pt}
 \vspace{-6pt}
\caption{Comparison between IFM and ACL(DS). Under standard linear evaluation IFM performs best. ACL is suited to adversarial evaluation.}
\label{fig: simadv acl comparison}
\end{adjustwidth}
\end{minipage}
\vspace{-12pt}
\end{figure}

 \begin{figure}[t] %{6.5cm}
  \begin{center}
    \includegraphics[width=0.32\textwidth]{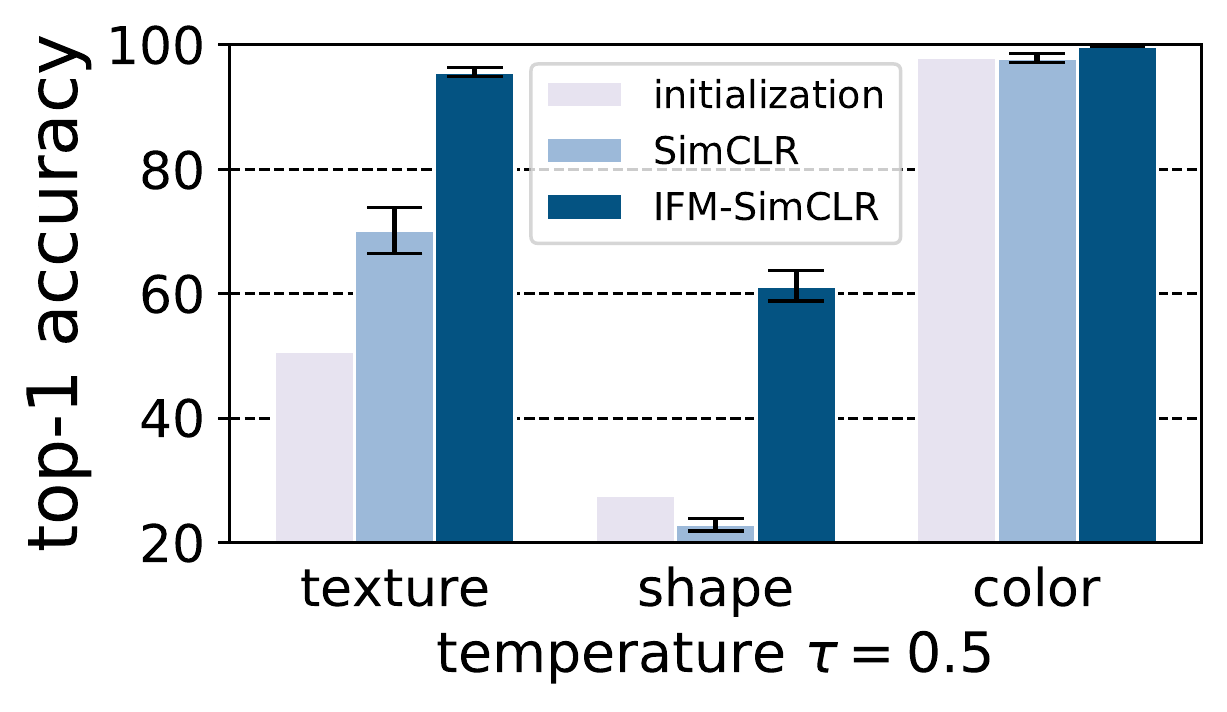}
    \includegraphics[width=0.32\textwidth]{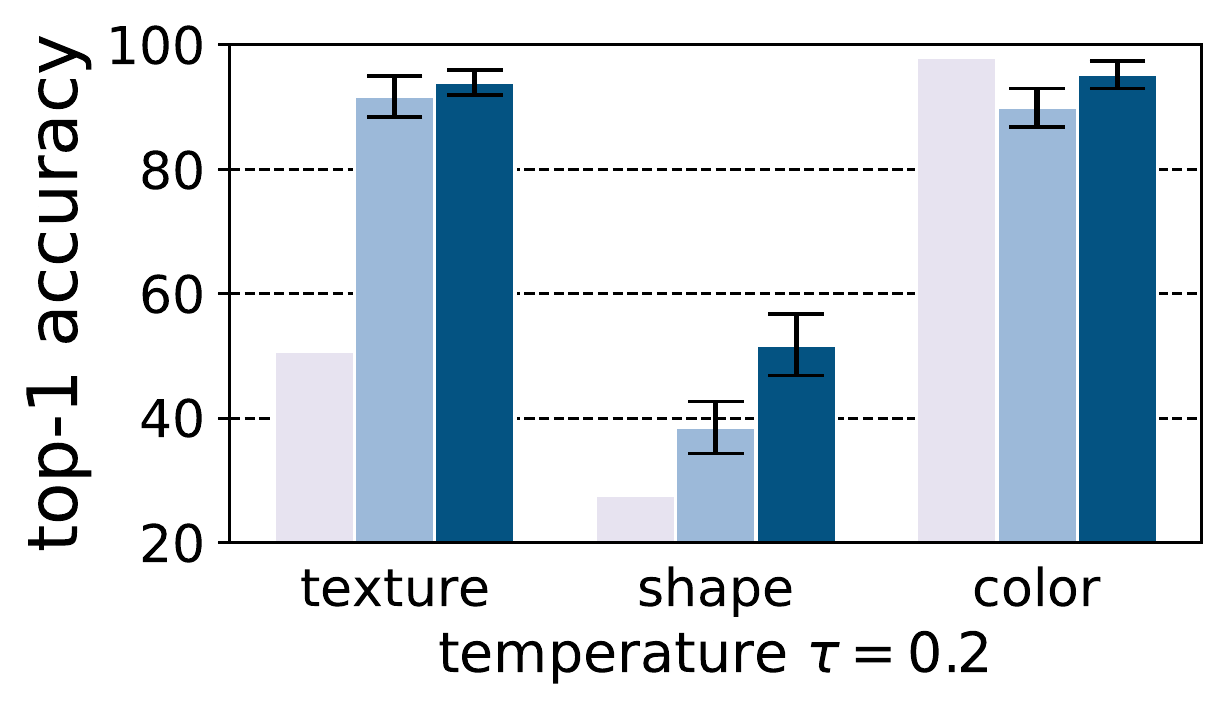}
    \includegraphics[width=0.32\textwidth]{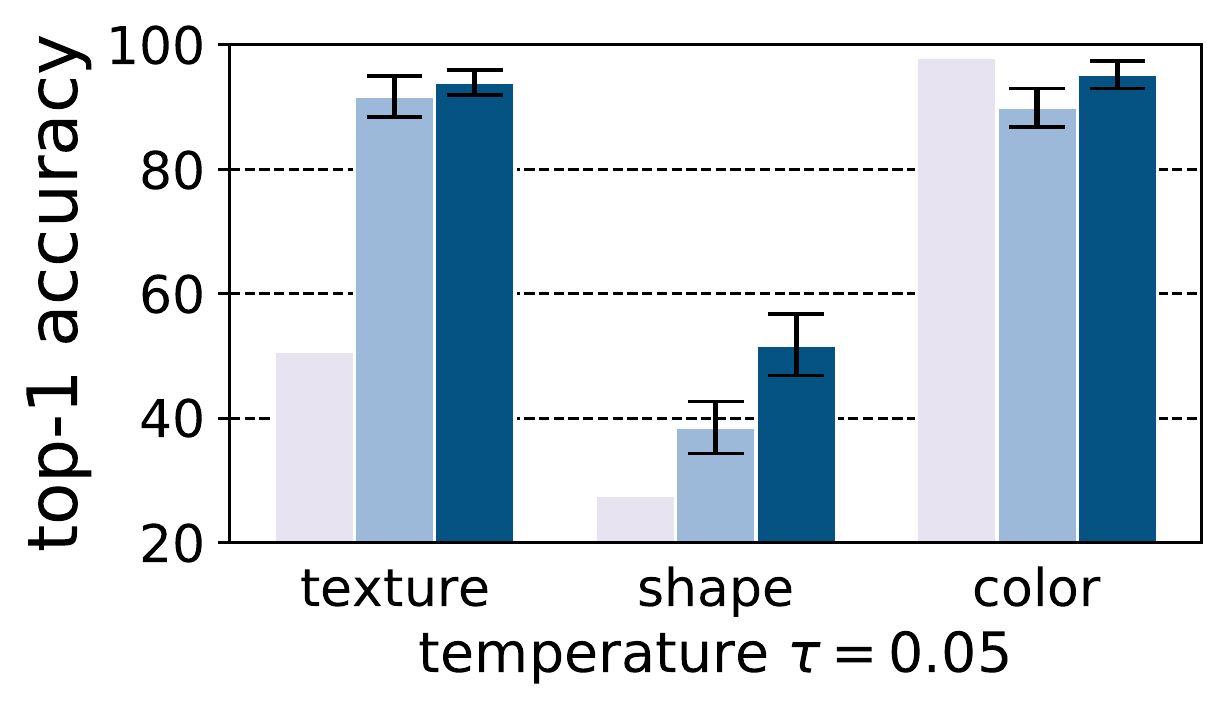}  
      \vspace{-5pt}
    \caption{\textbf{Trifeature dataset.} Implicit feature modification reduces feature suppression, enhancing the representation of texture, shape and color features simultaneously. All results  are average linear readout accuracy over three seeds and use a fixed value $\varepsilon=0.1$ to illustrate robustness to $\varepsilon$. }
    \label{fig: simadv suppression}
  \end{center}
  \vspace{-15pt}
\end{figure}

Implicit feature modification (IFM) can be used with any InfoNCE-based contrastive framework,  and we write IFM-SimCLR, IFM-MoCo-v2 etc. to denote IFM applied within a specific framework. Code for IFM will be released publicly, and is also available in the supplementary material.
\vspace{-5pt}

\subsection{Does implicit feature modification help avoid feature suppression?}\label{expt: visualization}
\vspace{-5pt}

We study the effect IFM has on feature suppression by training ResNet-18 encoders for 200 epochs with $\tau\in \{0.05, 0.2, 0.5\}$ on the Trifeature dataset  \cite{hermann2020shapes}. Results are averaged over three seeds, with IFM using $\varepsilon=0.1$ for simplicity. Fig. \ref{fig: simadv suppression} shows that IFM improves the linear readout accuracy across \emph{all} three features for all temperature settings. 
The capability of IFM to enhance the representation of all features -- i.e. reduce reliance on shortcut solutions -- is an important contrast with tuning temperature $\tau$ or using hard negatives, which  Fig. \ref{fig: feature suppression} shows only \emph{trades-off} which features are learned. 

\vspace{-5pt}
\subsection{Performance on downstream tasks}\label{sec: vis rep}
\vspace{-5pt}

Sec. \ref{sec: vis} and Sec. \ref{expt: visualization} demonstrate that implicit feature modification is adept at altering high-level features of an input, and combats feature suppression. This section shows that these desirable traits translate into improved performance on object classification and medical imaging tasks.
\vspace{-5pt}

\paragraph{Experimental setup for classification tasks.} Having observed the positive effect IFM has on feature suppression, we next test if this feeds through to improved performance on real tasks of interest.  We benchmark using both SimCLR and MoCo-v2 \cite{chen2020simple,chen2020mocov2} with standard data augmentation \cite{chen2020simple}. All encoders have ResNet-50 backbones and are trained for 400 epochs (with the exception of on ImageNet100, which is trained for 200 epochs). All encoders are evaluated using the test accuracy of a linear classifier trained on the full training dataset (see Appdx. \ref{appdx: Object classification experimental setup details} for full setup details).
\vspace{-5pt}

\begin{wraptable}{r}{10.0cm} %6.8cm
 \vspace{-10pt}
\begin{tabular}{cccccc}
\toprule  
-- & MoCo-v2 & AdCo \cite{hu2020adco} & \multicolumn{3}{c}{IFM-MoCo-v2}
 \\
    \cmidrule(lr){1-1} \cmidrule(lr){2-2} \cmidrule(lr){3-3} 
 \cmidrule(lr){4-6}
$\varepsilon$ & N/A & N/A & $0.05$ & $0.1$ & $0.2$ \\
\toprule  
top-1 & ${80.4}_{\pm 0.11}$ & ${78.9}_{\pm 0.21}$ & $\textbf{81.1}_{\pm 0.02}$ & $80.9_{\pm 0.25}$ & ${80.7}_{\pm 0.13}$ \\
%top-5 & 95.6\% & 94.5\% & \textbf{95.8\%} & 95.5\%  & 95.5\% \\
\bottomrule
\end{tabular}
\caption{Linear readout ($\%$) on ImageNet100, averaged over five seeds. IFM improves over MoCo-v2 for all settings of $\varepsilon$.}
 \vspace{-10pt}
 \label{tab: imagenet100}
\end{wraptable}

\vspace{-4pt}

\paragraph{Classification tasks.} Results given in Fig. \ref{fig: visual lin readout} and Tab. \ref{tab: imagenet100} find that every value of $0 < \varepsilon \leq 0.2$ \emph{improves performance across all datasets using both MoCo-v2 and SimCLR frameworks}. We find that optimizing $\mathcal L_\varepsilon$  ($76.0\%$ average score across all eight runs in Fig. \ref{fig: visual lin readout}) performs similarly to the standard contrastive loss ($75.9\%$ average score), and does worse than the IFM loss $(\mathcal L+ \mathcal L_\varepsilon)/2$. This suggests that  $\mathcal L$ and  $\mathcal L_\varepsilon$ learn  complementary features. Tab. \ref{tab: imagenet100} benchmarks IFM on ImageNet100 \cite{tian2019contrastive} using MoCo-v2, observing improvements of $0.9\%$. We also compare results on ImageNet100 to AdCo \cite{hu2020adco}, another adversarial method for contrastive learning. We adopt the official code and use the exact same training and finetuning hyperparameters as for MoCo-v2 and IFM. For the AdCo-specific hyperparamters -- negatives learning rate $lr_\text{neg}$ and negatives temperature $\tau_\text{neg}$ -- we use a grid search  over all combinations $lr_\text{neg} \in \{ 1,2,3,4\}$ and $\tau_\text{neg} \in \{ 0.02,0.1\}$, which includes the AdCo default ImageNet1K recommendations $lr_\text{neg} =3$ and $\tau_\text{neg}=0.02$ \cite{hu2020adco}. The resulting AdCo performance of $78.9\%$ is slightly below MoCo-v2. However using their respective ImageNet1K default parameters AdCo and MoCo-v2 achieve $72.4\%$ and $71.8\%$ respectively, suggesting that the discrepancy between AdCo and MoCo-v2 may in part be  due to the use of improved hyperparameters tuned on MoCo-v2. Note importantly, IFM is robust to the choice of $\varepsilon$: all values $\varepsilon \in \{0.05,0.1,0.2\}$ were found to boost performance across all datasets and all frameworks. We emphasize that the MoCo-v2 baseline performance of $80.5\%$ on ImageNet100 is strong. Our hyperparameters, which we detail in  Appdx. \ref{sec: MoCo-v2 for ImageNet100}, may be of  interest to other works benchmarking MoCo-v2 on ImageNet100.
\vspace{-5pt}

 \begin{figure}[t] %{6.5cm}
  \includegraphics[width=\textwidth]{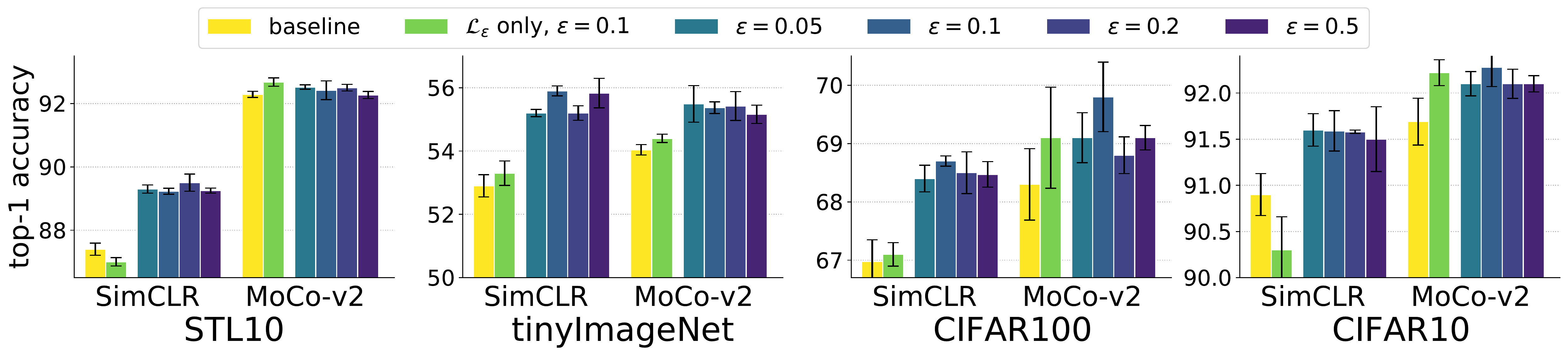}
  \vspace{-15pt}
     \caption{IFM improves linear readout performance on all datasets for all $\varepsilon \in \{0.05,0.1,0.2\}$ compared to  baselines. Protocol uses $400$ epochs of training with ResNet-50 backbone. }
      \label{fig: visual lin readout}
       \vspace{-15pt}
\end{figure}

\vspace{-4pt}
\paragraph{Medical images.}

To evaluate our method on a modality differing significantly from object-based images we consider the task of learning representations of medical images. We benchmark using the approach proposed by~\cite{sun2020context} which is a variant of MoCo-v2 that incorporates the anatomical context in the medical images. 
We evaluate our method on the COPDGene dataset~\cite{regan2011genetic}, which is a multi-center observational study focused on the genetic epidemiology of Chronic obstructive pulmonary disease (COPD). See Appdx. \ref{appdx: COPD} for full background details on the COPDGene dataset, the five  COPD related outcomes we use for evaluation, and our implementation. We perform regression analysis for continuous outcomes in terms of coefficient of determination (R-square), and logistic regression to predict ordinal outcomes and report the classification accuracy and the \textit{1-off} accuracy, i.e., the probability of the predicted category is within one class of true value.

 \begin{table*}[t]
 \begin{adjustbox}{max width=\textwidth}
   \centering
 %\vspace{-3mm}
  \begin{tabular}{lcccccccc}
  \toprule
    Method&
    $\log$\texttt{FEV1pp}&
    $\log$\texttt{$\text{FEV}_1\text{FVC}$}&
    CLE&
    CLE \textit{1-off}&
    Para-septal&
    Para-septal \textit{1-off}&
    mMRC&
    mMRC \textit{1-off}\\
  \cmidrule(lr){1-1} \cmidrule(lr){2-3} 
 \cmidrule(lr){4-9}
   Loss& \multicolumn{2}{c}{R-Square}& \multicolumn{6}{c}{ Accuracy (\%)}  
    \\
  \cmidrule(lr){1-1} \cmidrule(lr){2-3} 
 \cmidrule(lr){4-9}
 $\mathcal{L}$ (baseline) &$0.566_{\pm.005}$&$0.661_{\pm.005}$&$49.6_{\pm0.4}$&$81.8_{\pm0.5}$&$55.7_{\pm0.3}$&$84.4_{\pm0.2}$&$50.4_{\pm0.5}$&$72.5_{\pm0.3}$\\
  $\mathcal{L}_\varepsilon$, $\varepsilon=0.1$ &$0.591_{\pm.008}$&$0.681_{\pm.008}$&$49.4_{\pm0.4}$&$81.9_{\pm0.3}$&${55.6_{\pm0.3}}$&$85.1_{\pm0.2}$&$50.3_{\pm0.8}$&$72.7_{\pm0.4}$\\
   IFM, $\varepsilon=0.1$ &$\bm{0.615_{\pm.005}}$&$\bm{0.691_{\pm.006}}$&$48.2_{\pm0.8}$&$80.6_{\pm0.4}$&$55.3_{\pm0.4}$&$84.7_{\pm0.3}$&$50.4_{\pm0.5}$&$72.8_{\pm0.2}$\\
   IFM, $\varepsilon=0.2$ &$0.595_{\pm.006}$&${0.683_{\pm.006}}$&$48.5_{\pm0.6}$&${80.5_{\pm0.6}}$&$55.3_{\pm0.3}$&${85.1_{\pm0.1}}$&${49.8_{\pm0.8}}$&${72.0_{\pm0.3}}$\\
   IFM, $\varepsilon=0.5$ &${0.607_{\pm.006}}$&$0.683_{\pm.005}$&${49.6_{\pm0.4}}$&$82.0_{\pm0.3}$&$54.9_{\pm0.2}$&$84.7_{\pm0.2}$&$\bm{50.6_{\pm0.4}}$&$\bm{73.1_{\pm0.2}}$\\
   IFM, $\varepsilon=1.0$ &$0.583_{\pm.005}$&$0.675_{\pm.006}$&$\bm{50.0_{\pm0.5}}$&$\bm{82.9_{\pm0.4}}$&$\bm{56.3_{\pm0.6}}$&$\bm{85.7_{\pm0.2}}$&$50.3_{\pm0.6}$&$71.9_{\pm0.3}$\\
  %\vspace{-3mm}
   \bottomrule
   \end{tabular}
   \end{adjustbox}
   \caption{Linear readout performance on COPDGene dataset. The values are the average of 5-fold cross validation with standard deviations. The bold face indicates the best average performance. IFM yields improvements on all phenotype predictions.}
    \label{tbl:copd_epoch_10}
   \vspace{-10pt}
 \end{table*}

Tab.~\ref{tbl:copd_epoch_10} reports results. For fair comparison we use  same experimental configuration for the baseline approach~\cite{sun2020context} and our method. We find that IFM yields improvements on all outcome predictions. The gain is largest on spirometry outcome prediction, particularly $\text{logFEV1pp}$ with improvement of $8.7\%$ with $\varepsilon=0.1$.
%Although other settings of $\varepsilon$ achieve optimal performance on certain downstream tasks, 
We found that at least $\varepsilon=0.5$ and $1.0$ improve performance on all tasks. However, we note that not all features yield a statistically significant improvement with IFM.

 \begin{figure}[t] %{6.5cm}
  \includegraphics[width=\textwidth]{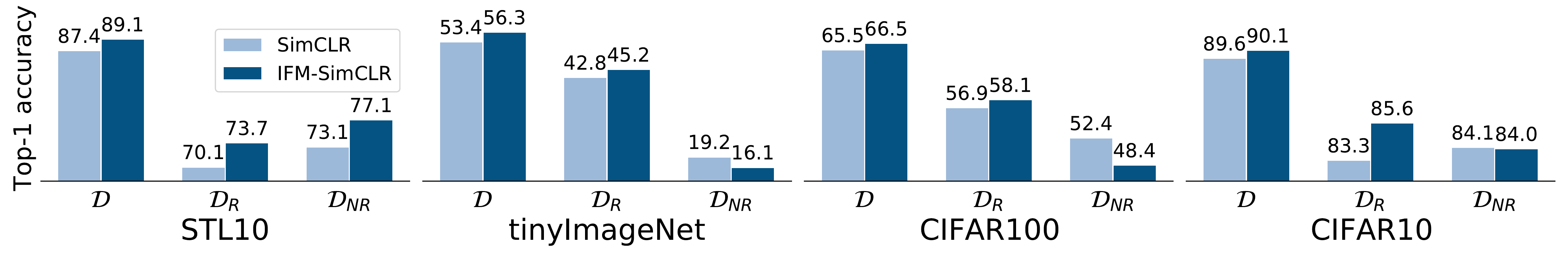}
  \vspace{-20pt}
     \caption{Label $\{\mathcal D, \mathcal D_\text{R},\mathcal D_\text{NR} \}$ indicates which dataset was used to train the linear readout function. Improved performance of IFM on standard data $\cal D$ can be attributed to improved representation of \emph{robust} features $\mathcal D_\text{R}$. See Sec. \ref{sec: robustnes} for construction of robust ($\mathcal D_\text{R}$) and non-robust ($\mathcal D_\text{NR}$) datasets. }
      \label{fig: robust features}
       \vspace{-10pt}
\end{figure}

%% file: understanding.tex
\vspace{-3pt}
\subsection{Further study on the impact of IFM on feature learning}\label{sec: robustnes}
\vspace{-10pt}

This section further studies the effect implicit feature modification has on \emph{what type} of features are extracted. Specifically, we consider the impact on learning of robust (higher-level) vs. non-robust features (pixel-level features). Our methodology, which is similar to that of Ilyas et al. \cite{ilyas2019adversarial} for deep supervised learning, involves carefully perturbing inputs to obtain non-robust features.

\vspace{-4pt}

\paragraph{Constructing non-robust features.} Given encoder $f$ we finetune a linear probe (classifier) $h$ on-top of $f$ using training data (we do not use data augmentation). Once $h$ is trained, we consider each labeled example $(x,y)$ from training data $\mathcal D_\text{train} \in \{ \text{tinyImageNet, STL10, CIFAR10, CIFAR100}\}$. A hallucinated target label $t$ is sampled uniformly at random, and we perturb $x=x_0$ until $h \circ f$ predicts $t$ using   repeated FGSM attacks  \cite{goodfellow2014explaining} $x_k \leftarrow x_{k-1} -\varepsilon \text{sign}( \nabla_x \ell( h \circ f (x_{k-1}), t))$. At each step we check if $\arg \max_i h \circ f(x_k)_i = t$ (we use the maximum of logits for inference) and stop iterating and set $x_\text{adv} = x_k$ for the first $k$ for which the prediction is $t$. This usually takes no more than a few FGSM steps with $\varepsilon=0.01$. We form a dataset of ``robust'' features by adding $(x_\text{adv},y)$ to $\mathcal D_R$, and a dataset of ``non-robust'' features by adding $(x_\text{adv},t)$ to $\mathcal D_{NR}$. To a human the pair $(x_\text{adv},t)$  will look mislabeled, but for the encoder $x_\text{adv}$ contains features predictive of $t$. Finally, we re-finetune (i.e. re-train)  linear classifier $g$ using $\mathcal D_{R}$ (resp. $\mathcal D_{NR}$).

 Fig. \ref{fig: robust features} compares accuracy of the re-finetuned models on a test set of \emph{standard} $\mathcal D_\text{test}$ examples (no perturbations are applied to the test set). Note that $\mathcal D_{R}$, $\mathcal D_{NR}$ depend on the original encoder $f$. When re-finetuning $f$ we always use datasets $\mathcal D_{R}$, $\mathcal D_{NR}$ formed via FGSM attacks on $f$ itself. So there is one set $\mathcal D_{R},\mathcal D_{NR}$ for SimCLR, and another set for IFM.  
 Fig. \ref{fig: robust features} shows that IFM achieves superior generalization ($\mathcal D$) compared to SimCLR \emph{by better representing robust features} ($\mathcal D_{R}$). Representation of non-robust  features ($\mathcal D_{NR}$) is similar for  IFM ($55.5\%$ average across all datasets) and SimCLR ($56.7\%$ average). IFM is juxtaposed to the supervised adversarial training of Madry et al., which \emph{sacrifices} standard supervised performance in exchange for not using non-robust features \cite{madry2017towards, tsipras2018robustness}.

%% file: discussion.tex
\vspace{-10pt}
\section{Discussion}\label{sec: discussion}
\vspace{-7pt}

This work studies the relation between contrastive instance discrimination and feature learning. While we focus specifically on contrastive learning, it would be of interest to also study feature learning for other  empirically successful self-supervised methods \cite{bardes2021VICReg,chen2020exploring,grill2020bootstrap,zbontar2021barlow}. Understanding differences in feature learning biases between different methods may inform which methods are best suited for a given task, as well  as point the way to further improved self-supervised techniques.

\paragraph{Acknowledgments} SJ was supported by NSF BIGDATA award IIS-1741341, NSF Convergence Accelerator Track D 2040636. SS acknowledges support from NSF-TRIPODS+X:RES (1839258). JR was partially supported by a Two Sigma fellowship. KB acknowledges support from NIH (1R01HL141813-01), NSF (1839332 Tripod+X), and a research grant from SAP SE Commonwealth Universal Research Enhancement (CURE) program awards research grants from the Pennsylvania Department of Health. Finally, we warmly thank Katherine Hermann and Andrew Lampinen for making the Trifeature dataset available for our use.

%% file: appendix.tex
\section{Proofs for Section \ref{sec: task design}}\label{appx: proofs}

In this section we give proofs for all the results in Sec. \ref{sec: task design}, which explores the phenomenon of feature suppression in contrastive learning using the InfoNCE loss. We invite the reader to consult Sec. \ref{sec: notarion setup} for details on any notation, terminology, or formulation details we use. 

Recall, for a measure $\nu$ on a space $\cal U$ and a measurable map $h : \cal U \rightarrow \cal V$ let $h \# \nu$ denote the \emph{pushforward} $h \# \nu(V)=\nu(h^{-1}(V))$ of a  measure $\nu$ on a space $\cal U$ for a measurable map $h : \cal U \rightarrow \cal V$ and  measurable $V \subseteq \cal V$, where $h^{-1}(V)$ denotes the preimage. We now recall the definition of feature suppression and distinction.

\begin{manualdef}{1}\label{def: appx def}
Consider an encoder $f :\mathcal X \rightarrow  \mathbb{S}^{d-1}$ and features $j \subseteq [n]$. For each $z^j \in %\in
\mathcal Z^S$, let $\mu (\cdot| z^S) = (f\circ g) \# \lambda (\cdot| z^j) $ be the pushforward measure  on $\mathbb{S}^{d-1}$ by $f\circ g$ of the conditional $\lambda (\cdot| z^S)$. 
\begin{enumerate}
    \vspace*{-7pt}   
    \setlength{\itemsep}{0pt}
    \item  $f$ \emph{suppresses} $S$ if for any pair $z^S , \bar{z}^S\in \mathcal Z^S$, we have $\mu (\cdot| z^S) =  \mu (\cdot| {\bar{z}^S})$. 
    %\vspace{-3pt}
    \item  $f$ \emph{distinguishes} $S$ if for any pair of distinct $z^S, \bar{z}^S\in \mathcal Z^S$, measures $\mu (\cdot| z^S), \mu (\cdot| {\bar{z}^S})$ have disjoint support. 
    \vspace{-5pt}
\end{enumerate}
\end{manualdef}

Suppression of features $S$ is thereby captured by the characteristic of distributing points in the same way on the sphere independently of what value $z^S$ takes. Feature distinction, meanwhile, is characterized by being able to partition the sphere into different pieces, each corresponding to a different value of $z^S$. Other (perhaps weaker) notions of feature distinction may be useful in other contexts. However here our goal is to establish that it is possible for InfoNCE optimal encoders both to suppress features in  the sense of Def. \ref{def: appx def}, but also to separate concepts out in a desirable manner. For this purpose we found this strong notion of distinguishing to suffice.  

Before stating and proving the result, recall the limiting InfoNCE loss that we analyze,
\vspace{-14pt}

\begin{equation*}
\mathcal L = \lim_{m \rightarrow \infty } \big \{ \mathcal L_m(f) - \log m  - \tfrac2{\tau}\big \} = \tfrac{1}{2\tau}\mathbb{E}_{x,x^+}  \| f(x) - f(x^+) \|^2 + \mathbb{E}_{x^+ } \log \big [  \mathbb{E}_{x^-} e^{f(x^+)^\top f(x^-)/\tau} \big ]. %$
\end{equation*}
%\vspace{-14pt}
We subtract $\log m$ to ensure the limit is finite, and use $x^-$ to denote a random sample with the same distribution as $x^-_i$. Following \cite{wang2020understanding} we denote the first term by $\mathcal L_{\text{align}}$ and the second term by the ``uniformity loss'' $\mathcal L_{\text{unif}}$, so $ \mathcal L=\mathcal L_{\text{align}}+\mathcal L_{\text{unif}}$. 

\iffalse
Recall that we use $\sigma_d$ denote the uniform distribution on $\mathbb{S}^{d-1}$. For any  $A, B \in \mathcal M(\mathbb{S}^{d-1}_\tau)$ let $\sigma_d(B|A) = \sigma_d( A \cap B) / \sigma_d(A)$ denote the probability of event $B$ under $\sigma_d$ conditioned on $A$.
\fi

\begin{manualprop}{1} Suppose that $p_j$ is uniform on $\mathcal Z^j = \mathbb{S}^{d-1}$. For any feature $j \in [n]$ there exists an encoder $f_\text{supp}$ that suppresses feature $j$ and encoder $f_\text{disc}$ that discriminates $j$ but both attain $\min_{f: \text{ measurable}}\mathcal L(f)$.
\end{manualprop}

\begin{proof}
The existence of the encoders $f_\text{supp}$ and $f_\text{disc}$  is demonstrated by constructing explicit examples. Before defining $f_\text{supp}$ and $f_\text{disc}$ themselves, we begin by constructing a family $\{f^k\}_{k\in [n]}$ of optimal encoders.

Since $g$ is injective, we know there exists a left inverse $h : \mathcal X \rightarrow \mathcal Z$ such that $h \circ g (z) = z$ for all $z \in \mathcal Z$. For any $k \in [n]$ let $\Pi^k : \mathcal Z \rightarrow \mathbb{S}^{d-1}$ denote the projection $\Pi^k(z)=z^k$. Since $p_k$ is uniform on the sphere $\mathbb{S}^{d-1}$, we know that $\Pi^k \circ h \circ g (z)= z^j$ is uniformly distributed on  $\mathbb{S}^{d-1}$. Next we partition the space $\mathcal X$. Since we assume that for all $a \neq a'$ and $z \neq z'$ that $a(z) \neq a'(z')$, the family $\{ \mathcal X_z\}_{z \in \mathcal Z}$ where $\mathcal X_z = \{ a\circ g(z) : z \in \mathcal Z\}$ is guaranteed to be a partition (and in particular, disjoint). We may therefore define an encoder $f_k: \mathcal X \rightarrow \mathbb{S}^{d-1}$ to be equal to $f_k(x) = \Pi^k  \circ h \circ g (z) = z^k$ for all $x \in \mathcal X_z$. 

First we check that this $f_k$ is optimal. Since for any $z$, and any $a \sim \mathcal A$, by definition we have $a\circ g(z) \in \mathcal X_z $, we have that $f_k(x) = f_k ( a(x))$ almost surely, so $\mathcal L_\text{align} (f_k) = 0$ is minimized. To show $f_k$ minimizes $\mathcal L_\text{unif}$ note that the uniformity loss can be re-written as 
\vspace{-10pt}

\begin{align*}
    \mathcal L_{\text{unif}}(f_k) &= \int_{a} \int_{z } \log \int_{a^- } \int_{z^-} e^{f_k\circ a(g(z)) ^\top f_k \circ a^-(g(z^-))/\tau} \lambda(\text{d}z)\lambda(\text{d}z^-) \mathcal A(\text{d}a)\mathcal A(\text{d}a^-)  \\
     &=  \int_{z } \log  \int_{z^-} e^{f_k\circ g(z) ^\top f_k \circ g(z^-)/\tau} \lambda(\text{d}z)\lambda(\text{d}z^-)  \\
    & = \int_{\mathbb{S}^{d-1}} \log  \int_{\mathbb{S}^{d-1}} e^{u ^\top v/\tau} \mu(\text{d}u)\mu(\text{d}v)
\end{align*}
where $\mu = f_k \circ g \# \lambda $ is the pushforward measure on $\mathbb{S}^{d-1}$, and the second equality follows from the fact that $\mathcal L_\text{align} (f_k) = 0$. Theorem 1 of Wang and Isola \cite{wang2020understanding} establishes that the operator, 
\vspace{-10pt}

\[ \mu \mapsto \int_{\mathbb{S}^{d-1}} \log  \int_{\mathbb{S}^{d-1}} e^{u ^\top v/\tau} \mu(\text{d}u)\mu(\text{d}v)\] 
\vspace{-10pt}

is minimized over the space of Borel measures on $\mathbb{S}^{d-1}$ if and only if $\mu = \sigma_d$, the uniform distribution on $\mathbb{S}^{d-1}$, as long as such an $f$ exists. However, since by construction $f_k(x) =  \Pi^k  \circ h \circ g (z) = z^k$ is uniformly distributed on $\mathbb{S}^{d-1}$, we know that $(f_k\circ g) \# \lambda= \sigma_d$, and hence that $f_k$ minimizes $ \mathcal L_\text{align}$ and $\mathcal L_\text{unif}$ and hence also the sum $\mathcal L = \mathcal L_\text{align} + \mathcal L_\text{unif}$. 

Recall that we seek encoder $f_\text{supp}$ that suppress feature $j$, and $f_\text{disc}$ that distinguishes feature $j$. We have a family $\{ f^k\}_{k \in [n]}$ that are optimal, and select the two encoders we seen from this collection. First, for $f_\text{supp}$ define $f_\text{supp} = f^k$ for any $k \neq j$. Then by construction $f_\text{supp}(x) = z^k$ (where $x \in \mathcal X_z$) depends only on $z^k$, which is independent of $z^j$. Due to independence, we therefore know that for any pair $z^j , \bar{z}^j\in \mathcal Z^j$, we have $\mu (\cdot| z^j) =  \mu (\cdot| {\bar{z}^j})$, i.e., that $f_\text{supp}$ is optimal but suppresses feature $j$. Similarly, simply define $f_\text{disc} = f^j$. So $f_\text{disc}(x) = z^j$ where $x \in \mathcal X_z$, and for any $z^j , \bar{z}^j\in \mathcal Z^j$ with $z^j \neq \bar{z}^j$ the pushforwards $\mu (\cdot| z^j),  \mu (\cdot| {\bar{z}^j})$ are the Dirac measures $\delta_{z^j},\delta_{\bar{z}^j}$, which are disjoint.
\end{proof}

 \begin{figure}[t] %{6.5cm}
  \includegraphics[width=\textwidth]{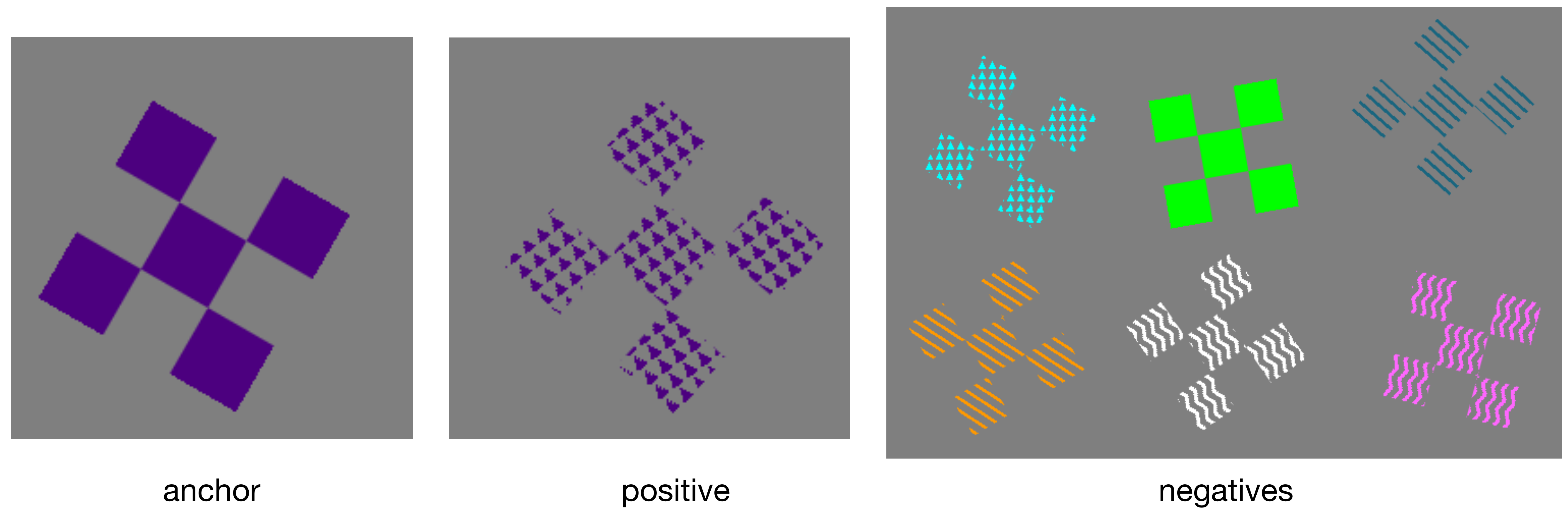}
     \caption{Visual illustration of Prop. \ref{prop: S held constant implies suppressed} using Trifeature samples \cite{hermann2020shapes}. In this example the shape feature is kept fixed across all positives and negatives ( $S = \{\text{shape}\}$), with color and texture to varying. As a consequence, the positive pair cannot be discriminated from negatives  using the shape feature. The encoder must learn color features in order to identify this positive pair. In other words, if a given set of features (e.g. $S = \{\text{shape}\}$) are constant across positive and negative pairs, then instance discrimination task demands the use of features in the compliment  (e.g. $\{\text{color,texture}\}$).}
       \vspace{-10pt}
\label{fig: trifeature second prop vis}
\end{figure}

Next we present a result showing that, under suitable conditions that guarantee that minimizers exists, any $f$ optimizing the InfoNCE loss is guaranteed to suppress features $S$ if all batches $x_1^+, x_2^+ , \{ x_i^-\}_{i=1}^N $ are  have the same features $S$ (but that the value $z_S$ taken is allowed to vary). This result captures the natural intuition that if a feature cannot be used to discriminate instances, then it will not be learned by the InfoNCE loss. Before reading the proposition, we encourage the reader to see Fig. \ref{fig: trifeature second prop vis} for an intuitive visual illustration of the idea underlying Prop. \ref{prop: S held constant implies suppressed} using Trifeature samples \cite{hermann2020shapes}.

However, this result also points to a way to manage which features are learned by an encoder, since if $f$ is guaranteed \emph{not} to learn features $S$, then necessarily $f$ must use other features to solve the instance discrimination task. This insight lays the foundation for the implicit feature modification technique, which perturbs the embedding $v=f(x)$ to remove information that $f$ uses to discriminate instances -- and then asks for instance discrimination using \emph{both} to original embedding, and the modified one -- with the idea that this encourages $f$ to learn new features that it previously suppressed.

\begin{manualprop}{2} For a set $S \subseteq [n]$ of features let 
\vspace{-10pt}

\begin{equation*}
\mathcal L_S(f) = \mathcal L_{\text{align}}(f) + \mathbb{E}_{x^+} \big [ - \log  \mathbb{E}_{x^-}  [ e^{f(x^+)^\top f(x^-)} |  z^S = z^{S-}  ] \big ] \end{equation*}
\vspace{-10pt}

denote the (limiting) InfoNCE conditioned on $x^+,x^-$ having the same features $S$.  Suppose that $p_j$ is uniform on $\mathcal Z^j = \mathbb{S}^{d-1}$ for all $j \in [n]$.
Then the infimum  $ \inf \mathcal L_S$ is attained, and every $f \in \min_{f'} \mathcal L_S(f')$  suppresses features $S$ almost surely.
\end{manualprop}
\vspace{-10pt}
\begin{proof}
By Prop \ref{prop: S held constant implies suppressed}, we know that for each $z^S$ there is a measurable $f$ such that $\mathcal L_{\text{align}}(f)=0$ and $f$ achieves perfect uniformity $(f\circ g) \# \lambda(\cdot| z^S) = \sigma_d$ conditioned on $z^S$. So consider such an $f$. Since $\mathcal L_{\text{align}}(f)=0$ we may write,
\vspace{-10pt}
\begin{align*}
  \mathcal L_S ( f) &= \mathbb{E}_{x^+} \big [ - \log  \mathbb{E}_{x^-}  [ e^{f(x^+)^\top f(x^-)} |  z^S = z^{S-}  ] \big ] \\
  &= \mathbb{E}_{z^S} \mathbb{E}_{z^{S-}} \big [ - \log  \mathbb{E}_{z^-}  [ e^{f \circ g(z)^\top f\circ g(z^-)} |  z^S = z^{S-}   ] \big ] \\
    &= \mathbb{E}_{z^S} \mathcal L ( f ; z^S) .
\end{align*}
\vspace{-10pt}

Where we have introduced the conditional loss function 
\vspace{-10pt}

\begin{align*}
\mathcal L ( f ; z^S)=  \mathbb{E}_{z^{S-}} \big [ - \log  \mathbb{E}_{z^-}  [ e^{f \circ g(z)^\top f\circ g(z^-)} |  z^S = z^{S-}   ] \big ] 
\end{align*}
\vspace{-10pt}

We shall show that any minimizer $f$ of $\mathcal L_S$ is such that $f$ minimizes $ \mathcal L ( f ; z^S) $ for all values of $z^S$. To show this notice that  $\min_f \mathcal L_S(f) = \min_f  \mathbb{E}_{z^S} \mathcal L ( f ; z^S) \geq  \mathbb{E}_{z^S}\min_f  \mathcal L ( f ; z^S)$ and if there is an $f$ such that $f$ minimizes $ \mathcal L ( f ; z^S) $ for each $z^S$ then the inequality is tight. So we make it our goal to show that there is an $f$ such that $f$ minimizes $ \mathcal L ( f ; z^S) $ for each $z^S$.

For fixed $z^S$, by assumption there is an $f_{z^S}$ such that    $(f_{z^S}\circ g) \# \lambda(\cdot| z^S) = \sigma_d$. That is, $f_{z^S}$  achieves perfect uniformity given $z^S$.  Theorem 1 of Wang and Isola \cite{wang2020understanding} implies that $f_{z^S}$ must minimize $\mathcal L ( f ; z^S)$. Given $\{f_{z^S}\}_{z^S}$ we construct an $f: \mathcal X \rightarrow \mathbb{S}^{d-1}_\tau $ that minimizes $\mathcal L ( f ; z^S)$ \emph{for all} $z^S$. By injectivity of $g$ we may partition $\mathcal X$ into pieces $\bigcup _{z^S \in \mathcal Z^S}\mathcal X_{z^S}$ where $\mathcal X_{z^S} = \{ x: x=g((z^S, z^{S^c})) \text{ for some } z^{S^c} \in \mathcal Z^{S^c}\}$. So we may simply define $f$ on domain $\mathcal X$ as follows:  $f(x) = f_{z^S}(x)$ if $x \in \mathcal X_{z^S}$.

This construction allows us to conclude that the minimum of $\mathcal L_S$ is attained, and any minimizer $f$ of $\mathcal L_S$ also minimizes $ \mathcal L ( f ; z^S) $ for each $z^S$. By Theorem 1 of Wang and Isola \cite{wang2020understanding} any such $f$ is such that $(f_{z^S}\circ g) \# \lambda(\cdot| z^S) = \sigma_d$ for all $z^S$, which immediately implies that $f$ suppresses features $S$.
\end{proof}

\section{Computation of implicit feature modification updates}\label{appdx: simadv}
\vspace{-10pt}
This section gives detailed derivations of two simple but key facts used in the development of IFM. The first result derives an analytic expression for the gradient of the InfoNCE loss with respect to positive sample in latent space, and the second result computes the gradient with respect to an arbitrary negative sample. The analysis is very simple, only requiring the use of elementary tools from calculus. Despite its simplicity, this result is very important, and forms the core of our approach. It is thanks to the analytic expressions for the gradients of the InfoNCE loss that we are able to implement our adversarial method \emph{without introducing any memory or run-time overheads}. This is a key distinction from previous adversarial methods for contrastive learning, which introduce significant overheads (see Fig. \ref{fig: simadv acl comparison}).

Recall the  statement of the lemma.

\begin{lemma}
For any $v,v^+, \{ v^-_i \}_{i=1}^m \in \mathbb{R}^d$ we have,
\vspace{-3pt}
\begin{equation*}
\nabla_{v^-_j} \ell  = \frac{e^{v^\top v^-_j}}{e^{v^\top v^+/\tau} + \sum_{i=1}^m e^{v^\top v^-_i/\tau}} \cdot \frac{v}{\tau}  \quad  \text{and} \quad  \nabla_{v^+} \ell   =  \bigg (   \frac{e^{v^\top v^+/\tau}}{e^{v^\top v^+/\tau} + \sum_{i=1}^m e^{v^\top v^-_i/\tau}} -1 \bigg ) \cdot \frac{v}{\tau} .
\end{equation*}
In particular, $\nabla_{v^-_j} \ell \propto v $ and $\nabla_{v^+} \ell \propto -v $.
\end{lemma}
\vspace{-10pt}

\begin{proof}
Both results follow from direct computation. First we compute $\nabla_{v^-_j} \ell(v,v^+, \{ v^-_i\}_{i=1}^m)$. Indeed, for any $j \in \{ 1, 2 , \ldots , m\} $ we have,

\begin{align*}
   \nabla_{v^-_j} \bigg \{  - \log \frac{e^{v^\top v^+/\tau}}{e^{v^\top v^+/\tau} + \sum_{i=1}^m e^{v^\top v^-_i/\tau}} \bigg \} &= \nabla_{v^-_j}    \log \bigg \{ e^{v^\top v^+/\tau} + \sum_{i=1}^m e^{v^\top v^-_i/\tau} \bigg \} \\
   &= \frac{\nabla_{v^-_j} \bigg \{ e^{v^\top v^+} + \sum_{i=1}^m e^{v^\top v^-_i/\tau} \bigg \}}{e^{v^\top v^+/\tau} + \sum_{i=1}^m e^{v^\top v^-_i/\tau}} \\
   &= \frac{e^{v^\top v^-_j/\tau} \cdot v/\tau}{e^{v^\top v^+/\tau} + \sum_{i=1}^m e^{v^\top v^-_i/\tau}}
\end{align*}

the quantity $\frac{e^{v^\top v^-_j/\tau} }{e^{v^\top v^+/\tau} + \sum_{i=1}^m e^{v^\top v^-_i/\tau}}>0$ is a strictly positive scalar, allowing us to conclude the derivative $\nabla_{v^-_j} \ell $ is proportional to  $ v $. We also compute $\nabla_{v^+} \ell(v,v^+, \{ v^-_i\}_{i=1}^m) $ in a similar fashion,

\begin{align*}
   \nabla_{v^+} \bigg \{  - \log \frac{e^{v^\top v^+/\tau}}{e^{v^\top v^+/\tau} + \sum_{i=1}^m e^{v^\top v^-_i/\tau}} \bigg \} &=  \nabla_{v^+} \big \{ - \log e^{v^\top v^+/\tau} \big \} + \nabla_{v^+}  \log \bigg \{ e^{v^\top v^+/\tau} + \sum_{i=1}^m e^{v^\top v^-_i/\tau} \bigg \} \\
   &= -\frac{v}{\tau} + \frac{ \nabla_{v^+} \bigg \{ e^{v^\top v^+/\tau} + \sum_{i=1}^m e^{v^\top v^-_i/\tau} \bigg \}}{e^{v^\top v^+/\tau} + \sum_{i=1}^m e^{v^\top v^-_i/\tau}} \\
   &= -\frac{v}{\tau} + \frac{e^{v^\top v^+/\tau} \cdot v/\tau}{e^{v^\top v^+/\tau} + \sum_{i=1}^m e^{v^\top v^-_i/\tau}} \\
   &=  \bigg (   \frac{e^{v^\top v^+/\tau}}{e^{v^\top v^+/\tau} + \sum_{i=1}^m e^{v^\top v^-_i/\tau}} -1 \bigg ) \cdot \frac{v}{\tau}.
\end{align*}

Since $0 < \frac{e^{v^\top v^+/\tau}}{e^{v^\top v^+/\tau} + \sum_{i=1}^m e^{v^\top v^-_i/\tau}} < 1$ we conclude in this case that the derivative $\nabla_{v^+} \ell $ points in the direction $-v$. 
\end{proof}

\subsection{ Alternative formulations of implicit feature modification}\label{appdx: pre-normalization}

This section contemplates two simple modifications to the IFM method with the aim of confirming that these modifications do not yield superior performance to the default proposed method. The two alternate methods focus around the following observation: IFM perturbs embeddings of unit length, and returns a modified version that will no longer be of unit length in general.
We consider two alternative variations of IFM that yield normalized embeddings. The first is the most simple solution possible: simply re-normalize perturbed embeddings to have unit length. The second is slightly more involved, and involves instead applying perturbations \emph{before} normalizing the embeddings. Perturbing unnormalized embeddings, then normalizing,  guarantees the final embeddings have unit length. The key property we observed in the original formulation was the existence of an analytic, easily computable closed form expressions for the derivatives. This property enables efficient computation of newly synthesized ``adversarial'' samples in latent space. Here we derive corresponding formulae for the pre-normalization attack.

For clarity, we introduce the slightly modified setting in full detail. We are given positive pair $x,x^+$ and a batch of negative samples $\{x^-_i\}_{i=1}^m$ and denote their encodings via $f$ as $v=f(x), v^+=f(x^+)$, and $v^-_i=f(x^-_i)$ for $i=1,\ldots m$ where we \emph{do not} assume that $f$ returns normalized vectors. That is, $f$ is allowed to map to anywhere in the ambient latent space $\mathbb{R}^d$.  The re-parameterized point-wise contrastive loss for this batch of samples is 
\vspace{-10pt}

\begin{equation*}
  \ell(v,v^+, \{ v^-_i\}_{i=1}^m) = - \log \frac{e^{\text{sim}(v, v^+)/\tau}}{e^{\text{sim}(v, v^+)/\tau} + \sum_{i=1}^m e^{\text{sim}(v, v^-_i)/\tau}},
\end{equation*}
\vspace{-10pt}

where $\text{sim}(u,v) = u \cdot v / \norm{u} \norm{v}$ denotes the cosine similarity measure. As before we wish to perturb $v^+$ and negative encodings $v^-_j$ to increase the loss, thereby making the negatives harder. Specifically we wish to solve $\max_{\delta^+ \in \mathcal B_{\varepsilon^+}, \{\delta_i^- \in \mathcal B_{\varepsilon_i} \}_{i=1}^m} \ell(v,v^+ + \delta^+, \{ v^-_i + \delta^-_i\}_{i=1}^m)$. The following lemma provides the corresponding gradient directions.

\begin{lemma}\label{lem: pre-norm derivative}
For any $v,v^+, \{ v^-_i\}_{i=1}^m \in \mathbb{R}^d$ we have 
\begin{equation*}
\nabla_{v^-_j} \ell \propto \frac{v}{  \norm{v}} - \text{sim}(v^-_j,v) \frac{v^-_j}{\|v^-_j\|} \quad  \text{and} \quad   \nabla_{v^+} \ell  \propto \frac{v}{  \norm{v}} - \text{sim}(v^+,v) \frac{v^+}{ \norm{v^+}}.
\end{equation*}
\end{lemma}

To prove this lemma we rely on the following  well-known closed form expression for the derivative of the cosine similarity, whose proof we omit.

\begin{lemma}
 $\nabla_{v} \text{sim}(v, u) = \frac{u}{ \norm{v} \norm{u}} - \text{sim}(v,u) \frac{v}{ \norm{v}^2}.$
\end{lemma}
\vspace{-10pt}

\begin{proof}[Proof of Lemma \ref{lem: pre-norm derivative}]
We compute, 
\vspace{-20pt}

\begin{align*}
    \nabla_{v^-_j} \ell &=  \nabla_{v^-_j}  \log \bigg ( e^{\text{sim}(v, v^+)} + \sum_{i=1}^m e^{\text{sim}(v, v^-_i)} \bigg ) \\
    &= \frac{e^{\text{sim}(v, v^-_j)}}{e^{\text{sim}(v, v^+)} + \sum_{i=1}^m e^{\text{sim}(v, v^-_i)}} \cdot \nabla_{v^-_j} \text{sim}(v, v^-_j)
\end{align*}
\vspace{-10pt}

Using the formula for the derivative of the cosine similarity, we arrive at a closed form formula, 
\begin{align*}
    \nabla_{v^-_j} \ell 
    &= \frac{e^{\text{sim}(v, v^-_j)}}{e^{\text{sim}(v, v^+)} + \sum_{i=1}^m e^{\text{sim}(v, v^-_i)}} \cdot \bigg (  \frac{v}{ \|v^-_j\| \norm{v}} - \text{sim}(v^-_j,v) \frac{v^-_j}{ \|v^-_j\|^2}) \bigg ). \\
    &\propto \frac{v}{  \|v\|} - \text{sim}(v^-_j,v) \frac{v^-_j}{ \|v^-_j\|}
\end{align*}
%\vspace{-20pt}
%
Similar computations yield
%\vspace{-10pt}
%
\begin{align*}
     \nabla_{v^+} \ell  &= - \nabla_{v^+}  \log \frac{e^{\text{sim}(v, v^+)}}{e^{\text{sim}(v, v^+)} + \sum_{i=1}^m e^{\text{sim}(v, v^-_j)}} \\
    &=\nabla_{v^+} \bigg ( - \text{sim}(v, v^+) + \log \big ( e^{\text{sim}(v, v^+)} + \sum_{i=1}^m e^{\text{sim}(v, v^-_i)} \big ) \bigg )   \\
    &=  \bigg (   \frac{e^{\text{sim}(v, v^+)}}{e^{\text{sim}(v, v^+)} + \sum_{i=1}^m e^{\text{sim}(v, v^-_i)}} -1 \bigg ) \cdot \nabla_{v^+} \text{sim}(v, v^+) \\
    &= \bigg (   \frac{e^{\text{sim}(v, v^+)}}{e^{\text{sim}(v, v^+)} + \sum_{i=1}^m e^{\text{sim}(v, v^-_i)}} -1 \bigg )\cdot \bigg (  \frac{v}{ \norm{v^+} \norm{v}} - \text{sim}(v^+,v) \frac{v^+}{ \norm{v^+}^2}) \bigg ) \\
    &\propto \frac{v}{  \norm{v}} - \text{sim}(v^+,v) \frac{v^+}{ \norm{v^+}}
\end{align*}
%\vspace{-10pt}
\end{proof}

Lemma \ref{lem: pre-norm derivative} provides precisely the  efficiently computable formulae for the derivatives we seek. One important difference between this pre-normalization case and the original setting is that the direction vector depends on $v^-_j$ and  $v^+$ respectively. In the original (unnormalized) setting the derivatives depend only on $v$, which allowed the immediate and exact discovery of the worst case perturbations in an $\varepsilon$-ball. Due to these additional dependencies in the pre-normalized case the optimization is more complex, and must be approximated iteratively. Although only approximate, it is still computationally cheap since we have simple analytic expressions for gradients.  

It is possible give an interpretation to the pre-normalization derivatives $\nabla_{v^-_j}\ell$ by considering the $\ell_2$ norm,
\vspace{-10pt}

\begin{align*}
\| \nabla_{v^-_j}\|_{2} &= \sqrt{\big (\frac{v}{\norm{v}}-\frac{v^\top v^-_j}{\norm{v}\|v^-_j\|}\frac{v^-_j}{\|v^-_j\|}\big )\cdot \big (\frac{v}{\norm{v}}-\frac{v^\top v^-_j}{\|v\|\|v^-_j\|}\frac{v^-_j}{\|v^-_j\|} \big )}\\
&=\sqrt{1+\text{sim}(v, v^-_i)^2-2\text{sim}(v, v^-_i)^2}\\
&=\sqrt{1-\text{sim}(v, v^-_i)^2}
\end{align*}
\vspace{-10pt}

So, samples $v^-_i$ with higher cosine similarity with anchor $v$ receive smaller updates. Similar calculations for $v^+$ show that higher cosine similarity with anchor $v$ leads to larger updates. In other words, the pre-normalization version of the method  automatically adopts an adaptive step size based on sample importance.

%\vspace{-15pt}
\subsubsection{Experimental results using alternative formulations}
%\vspace{-10pt}

In this section we test the two alternative implementations to confirm that these simple alternatives do not obtain superior performance to IFM. We consider only object-based images, so it remains possible that other modalities may benefit from alternate formulations. First note that $f$ encodes all points to the boundary of the same hypersphere, while perturbing $v^-_i \leftarrow v^-_i + \varepsilon_i   v  $ and $v^+ \leftarrow v^+ - \varepsilon_+   v  $ moves adversarial samples off this hypersphere. We therefore consider normalizing  all points again \emph{after} perturbing (\emph{+ norm}). The second method considers applying attacks \emph{before} normalization (\emph{+ pre-norm}), whose gradients were computed in the Lem. \ref{lem: pre-norm derivative}. It is still possible to compute analytic gradient expressions in this setting; we refer the reader to Appendix \ref{appdx: pre-normalization} for full details and derivations. Results reported in Tab. \ref{tab: alternative adv}, suggest that all versions improve over MoCov2, and both alternatives perform comparably to the default implementation based on Eqn. \ref{eqn: attack both loss}. 

\begin{table}
\centering
\begin{tabular}{ccccc}
\toprule  
Dataset & MoCo-v2 & \multicolumn{3}{c}{IFM-MoCo-v2}
 \\
  \cmidrule(lr){1-1}  \cmidrule(lr){2-2} 
 \cmidrule(lr){3-5}
-- &  -- &  default & + norm & + pre-norm \\
\toprule
STL10 & 92.4\% & 92.9\% & 92.9\% & \textbf{93.0\%} \\
CIFAR10 & 91.8\% & \textbf{92.4\%} & 92.2\% & 92.0\% \\
CIFAR100 & 69.0\% & \textbf{70.3\%} & 70.1\%  & 70.2\% \\
\bottomrule
\end{tabular}
\caption{Linear readout performance of alternative latent space adversarial methods. We report the best performance over runs for $\varepsilon \in \{0.05,0.1,0.2,0.5\}$. We find that the two modifications to IFM we considered do not improve performance compared to the default version of IFM.}
 \label{tab: alternative adv}
 \end{table}

\section{Supplementary experimental results and details}\label{appdx: experiments}

\subsection{Hardware and setup}

Experiments were run on two internal servers. The first consists of 8 NVIDIA GeForce RTX 2080 Ti GPUs (11GB). The second consists of 8 NVIDIA Tesla V100 GPUs (32GB). All experiments use the PyTorch deep learning framework \cite{paszke2019pytorch}. Specific references to pre-existing code bases used are given in the relevant sections below.

\subsection{Feature suppression experiments}\label{appdx: stl-digits suppression}

This section gives experimental details for all experiments in Sec. \ref{sec: task design} in the main manuscript, the section studying the relation between feature suppression and instance discrimination.

\subsubsection{Datasets}

 \begin{figure}[t] %{6.5cm}
  \includegraphics[width=\textwidth]{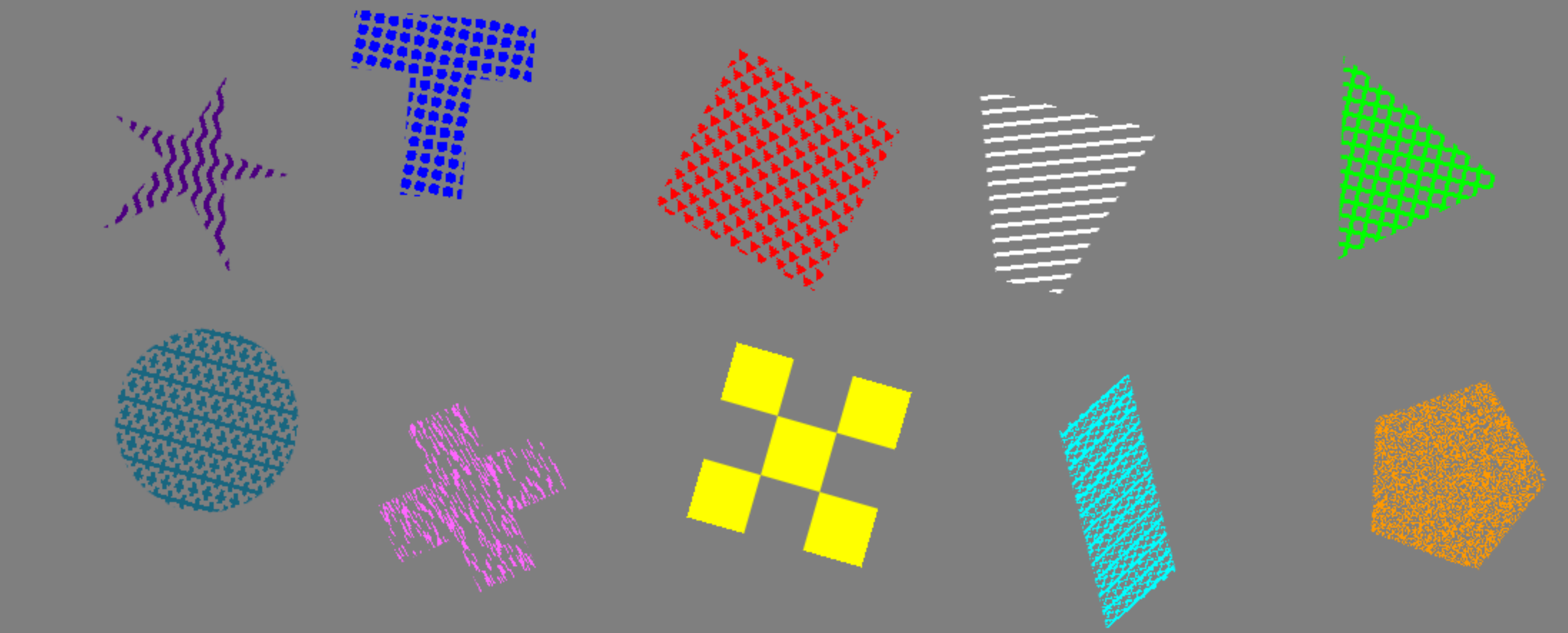}
     \caption{Sample images from the Trifeature dataset \cite{hermann2020shapes}. There are three features: shape, color, and texture. Each feature has $10$ different possible values. We show exactly one example of each feature.}
       \vspace{-10pt}
\label{fig: trifeature samples}
\end{figure}

 \begin{figure}[t] %{6.5cm}
  \includegraphics[width=\textwidth]{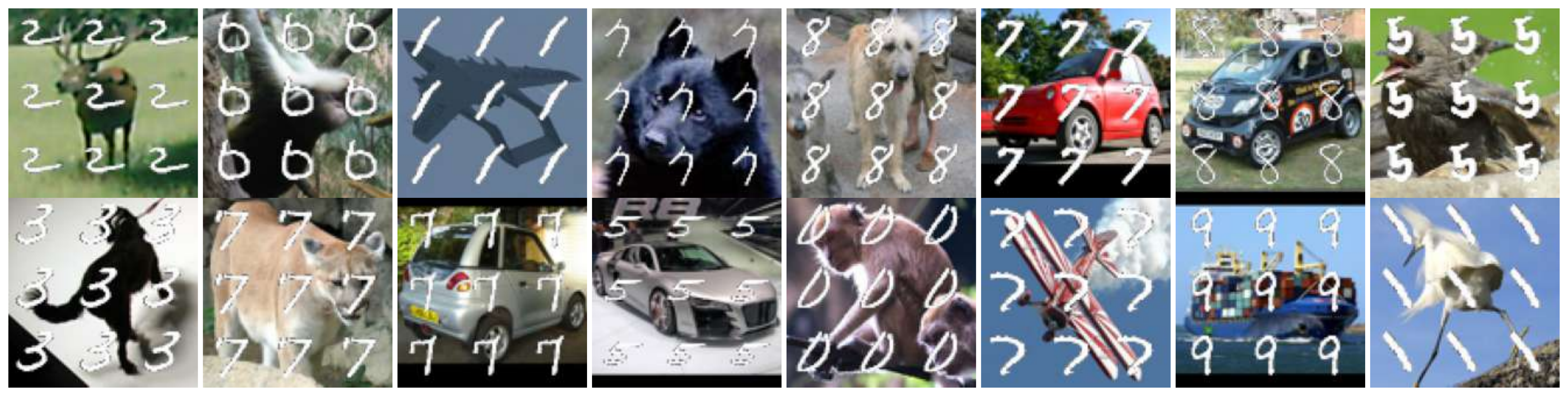}
     \caption{Sample images from the STL-digits dataset. There are two features: object class, and MNIST digit. Both features have $10$ different possible values.}
       \vspace{-10pt}
\label{fig: stl-digits samples}
\end{figure}

\paragraph{Trifeature \cite{hermann2020shapes}} Introduced by Hermann and Lampinen, each image is $128 \times 128$ and has three features: color, shape, and texture each taking 10 values. For each (color, shape, texture) triplet (1000 in total) Trifeature contains  100 examples, forming  a dataset of 100K examples in total. Train/val sets are obtained by a random 90/10 split. See Fig. \ref{fig: trifeature samples}, Appdx. \ref{appdx: experiments} for sample images.

\paragraph{STL10-digits dataset} We artificially combine MNIST digits and STL10 object to produce data with two controllable semantic features. We split the STL10 image into a $3 \times 3$ grid, placing a copy of the MNIST digit in the center of each sector. 
This is done by masking all MNIST pixels with intensity lower than $100$, and updating non-masked pixels in the STL10 image with the corresponding MNIST pixel value. 

\subsubsection{Experimental protocols}

\paragraph{Training} We train ResNet-18 encoders using SimCLR with batch size $512$. We use standard data SimCLR augmentations  \cite{chen2020simple}, but remove grayscaling and color jittering  when training on Trifeature in order to avoid corrupting color features. We use Adam optimizer, learning rate $1\times10^{-3}$ and weight decay $1\times10^{-6}$. Unless stated otherwise, the temperature $\tau$ is set to $0.5$.

\paragraph{Linear evaluation} For fast linear evaluation we first extract features from the trained encoder (applying the same augmentations to inputs as used during pre-training) then use the \texttt{LogisticRegression} function in scikit-learn \cite{pedregosa2011scikit} to train a linear classifier. We use the Limited-memory Broyden–Fletcher–Goldfarb–Shanno algorithm with a maximum iteration of $500$ for training. 

\subsubsection{Details on results}

\paragraph{Correlations Fig. \ref{fig: trifeature corr} }  For the Trifeature heatmap 33 encoders are used to compute correlations. The encoders are precisely encoders used to plot Fig. \ref{fig: feature suppression}. Similarly, the 7 encoders used to generate the STL-digits heatmap are precisely the encoders whose training is shown in Fig. \ref{fig: stl10 appdx}. When computing the InfoNCE loss for Fig. \ref{fig: trifeature corr}, for fair comparison all losses are computed using temperature normalization value $\tau=0.5$. This is independent of training, and is necessary only in evaluation to ensure loss values are comparable across different temperatures.

 \begin{figure}[t] %{6.5cm}
 \vspace{-5pt}
  \includegraphics[width=\textwidth]{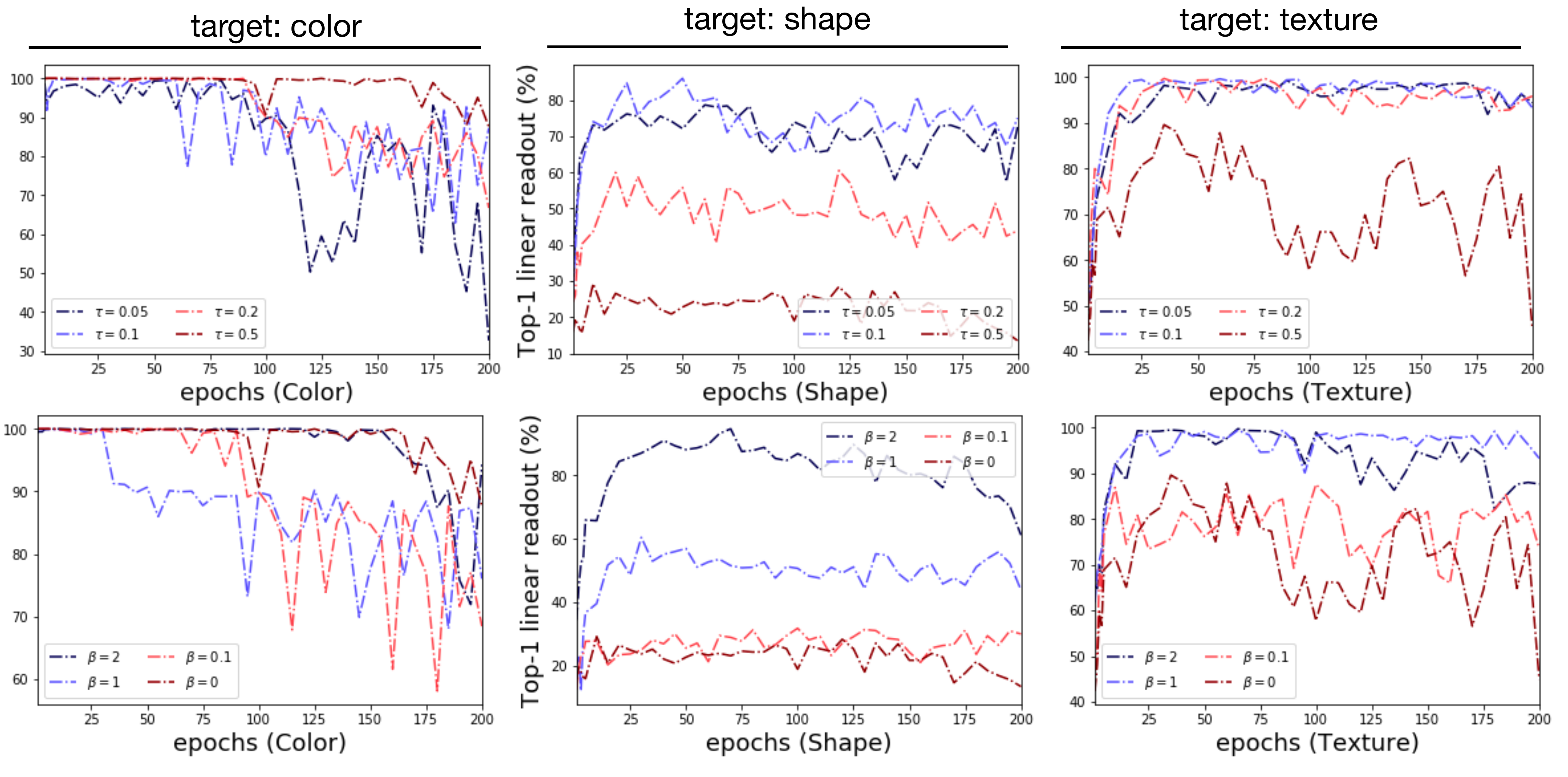}
  \vspace{-20pt}
\caption{Single run experiments showing training dynamics of Trifeature contrastive training. Linear readout performance on color prediction is particularly noisy.  }
\label{fig: trifeature stl10 appdx}
\end{figure}

 \begin{figure}[ht]
  \centering
  \includegraphics[width=65mm]{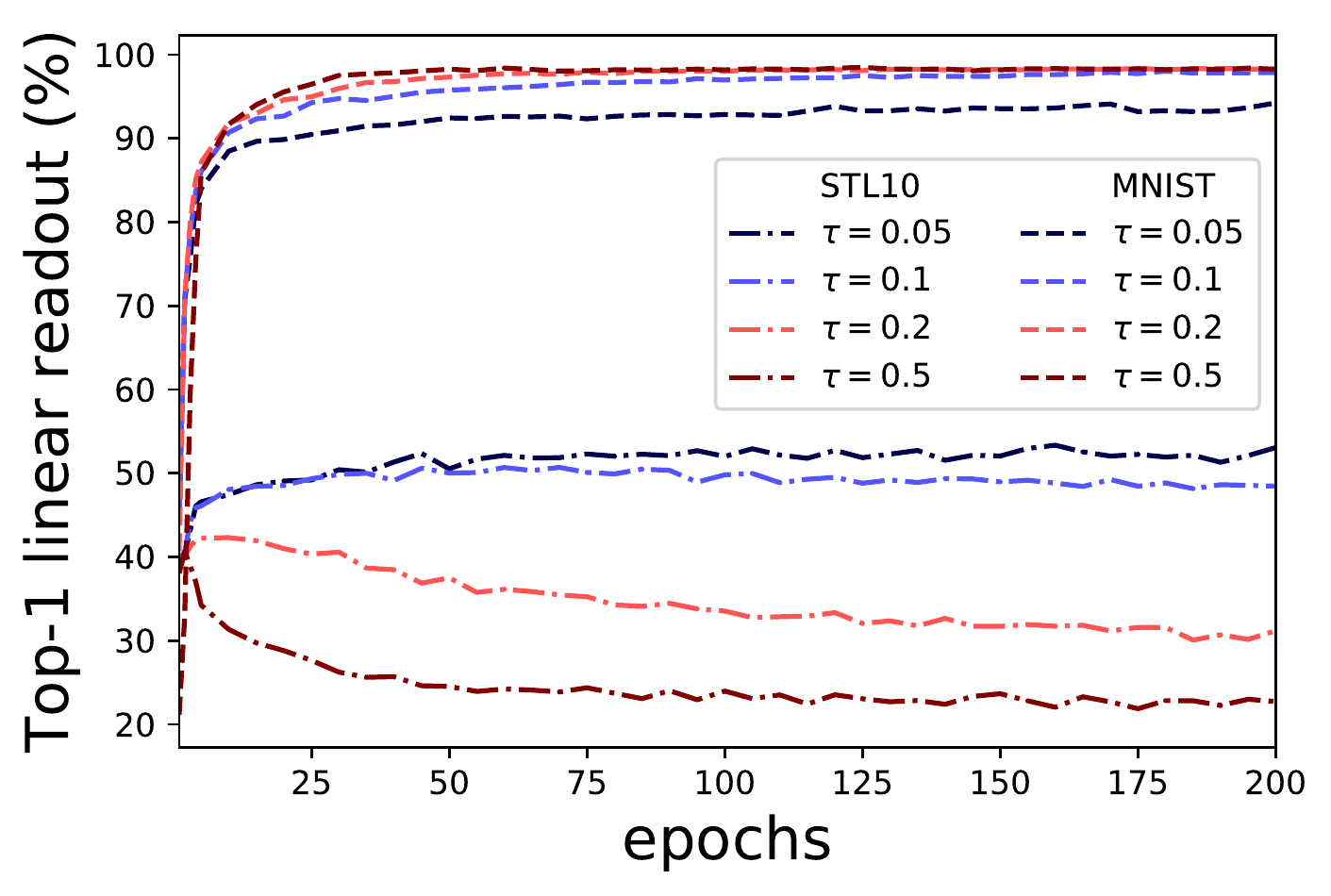}
  \includegraphics[width=65mm]{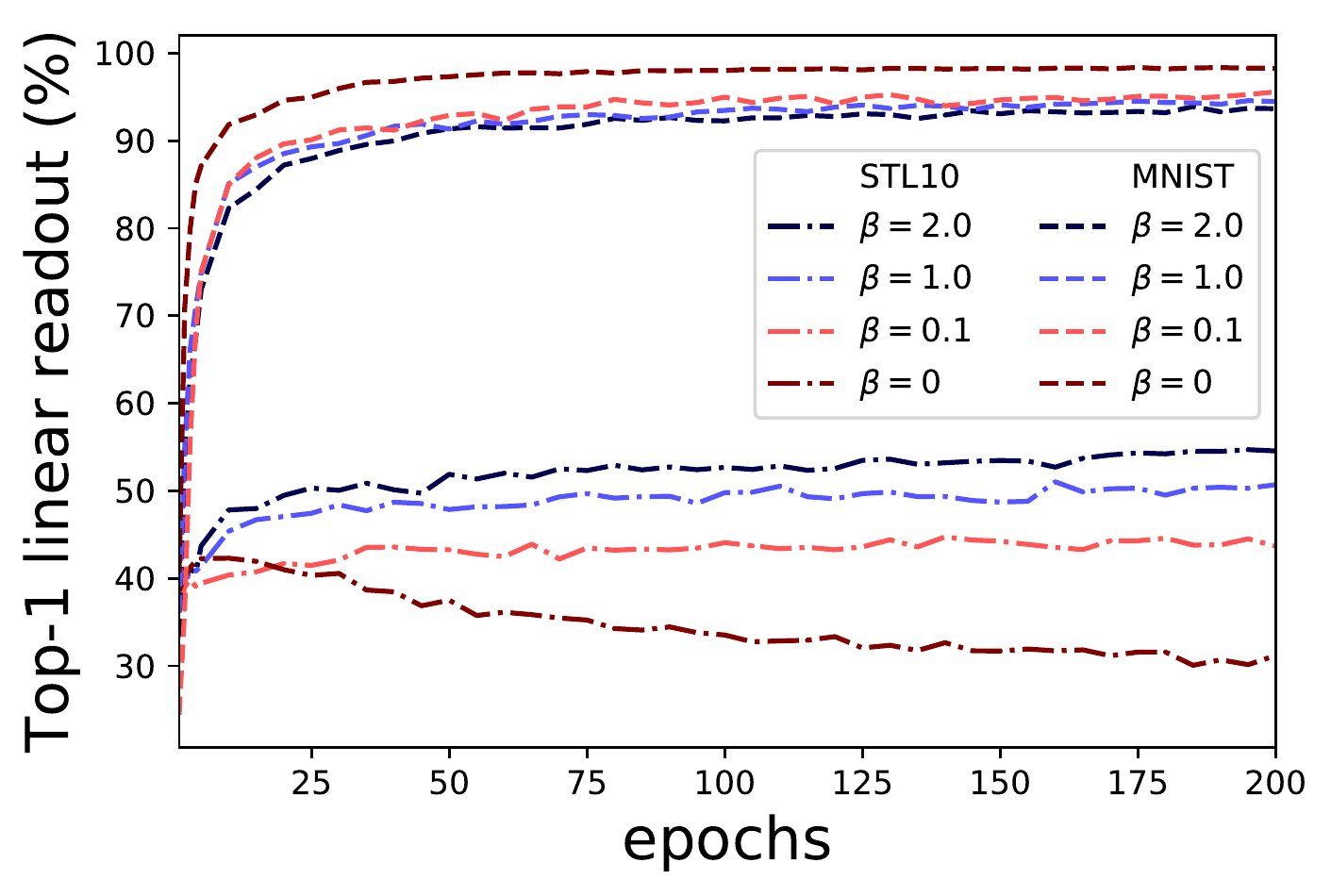}           
 \vspace{-3pt}
    \caption{STL-digits dataset. \textbf{Left:} performance on STL10 and MNIST linear readout for different temperature $\tau$ values. \textbf{Right:} performance on STL10 and MNIST linear readout for different hardness concentration $\beta$ values \cite{robinson2020contrastive}. In both cases harder instance discrimination (smaller $\tau$, bigger $\beta$) improves STL10 performance at the expense of MNIST. When  instance discrimination is too easy (big $\tau$, small $\beta$) STL10 features are \emph{suppressed}: achieving worse linear readout after training than at initialization. }%
        \label{fig: stl10 appdx}
         \vspace{-10pt}
 \end{figure}
 
Fig. \ref{fig: stl10 appdx} displays results for varying instance discrimination difficult on the STL-digits dataset. These results are complementing the Trifeature results in Fig. \ref{fig: feature suppression} in Sec. \ref{sec: task design} in the main manuscript. For STL-digits we report only a single training run per hyperparameter setting since performance is much more stable on STL-digits compared to Trifeature (see Fig. \ref{fig: trifeature stl10 appdx}). See Sec. \ref{sec: task design} for discussion of STL-digits results, which are qualitatively the same as on Trifeature. Finally, Fig. \ref{fig: simadv stldig} shows the effect of IFM on encoders trained on STL-digits. As with Trifeature, we find that IFM improves the performance on suppressed features (STL10), but only slightly. Unlike hard instance discrimination methods, IFM does not harm MNIST performance in the process. 

 \begin{figure}[t] %{6.5cm}
 \vspace{-7pt}
  \begin{center}
    \includegraphics[width=0.32\textwidth]{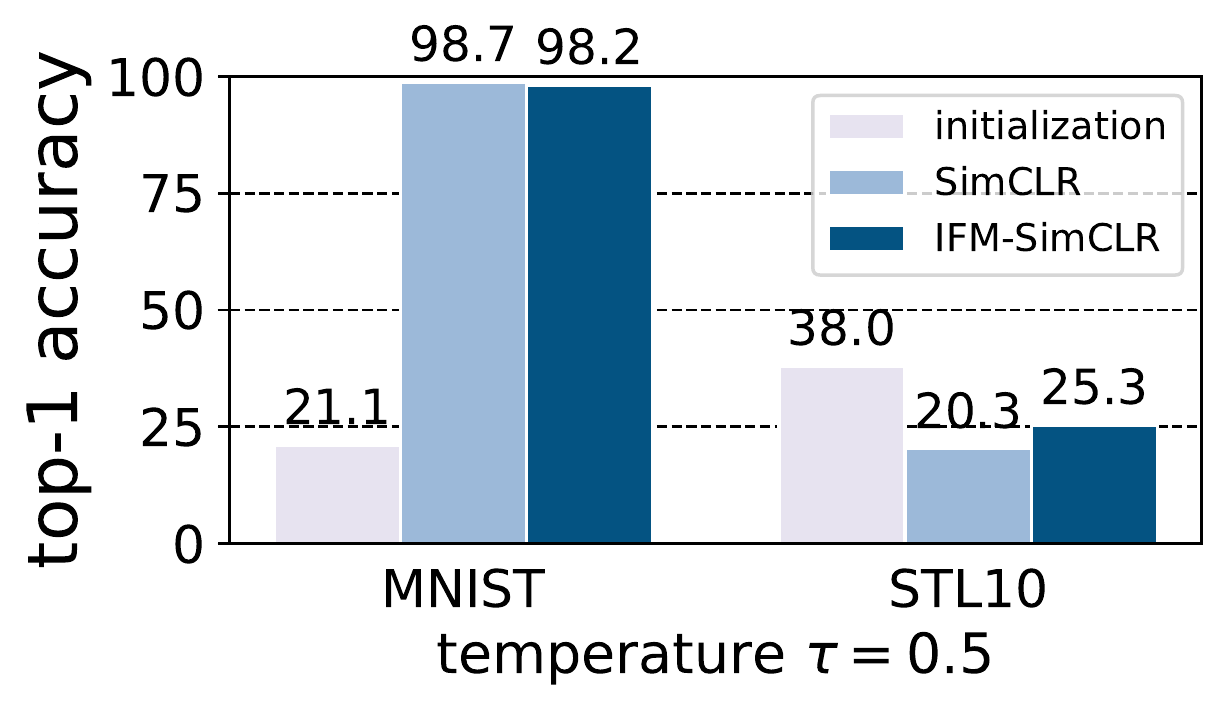}
    \includegraphics[width=0.32\textwidth]{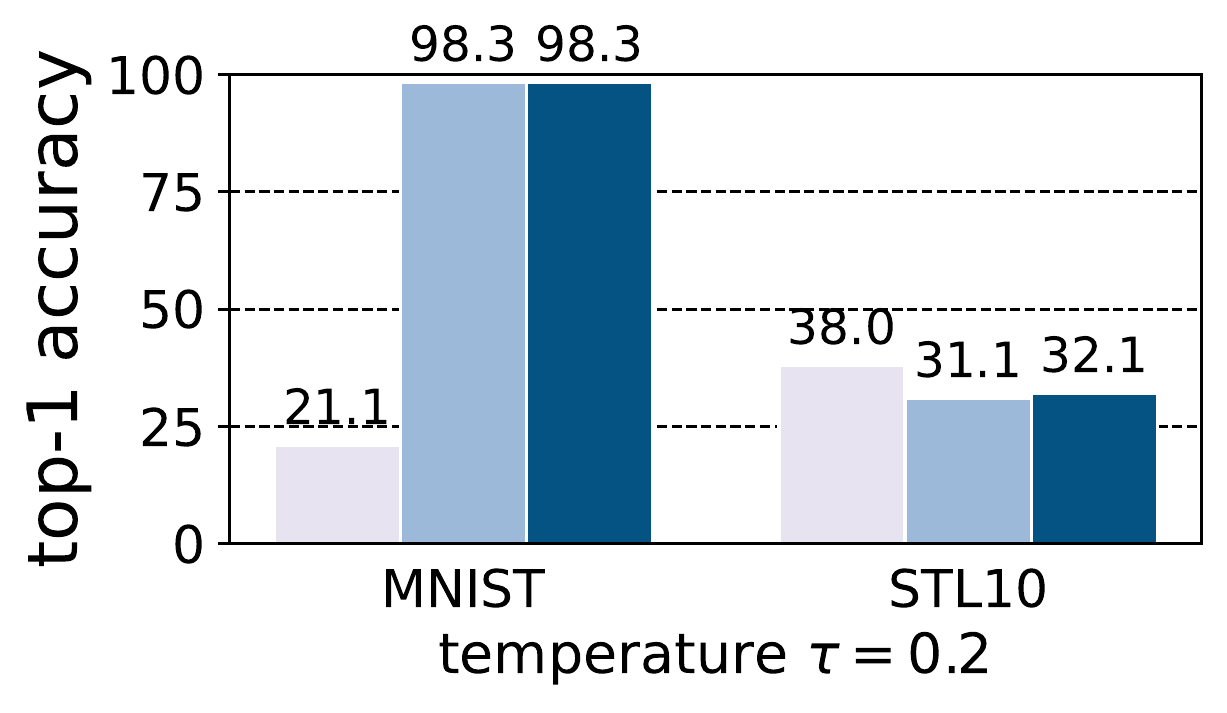}
    \includegraphics[width=0.32\textwidth]{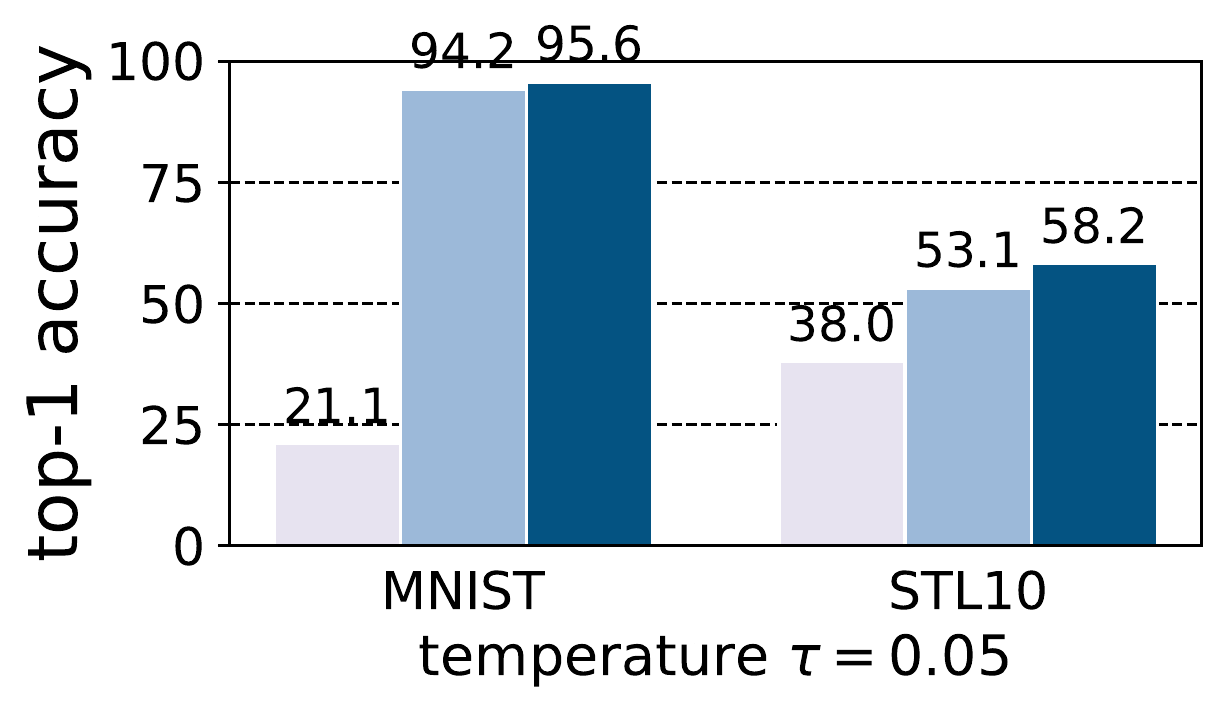}
      \vspace{-5pt}
    \caption{STL-digits dataset. Implicit feature modification reduces feature suppression, enhancing the representation of both MNIST and STL10 features simultaneously. All IFM runs use a fixed value $\varepsilon=0.1$, and loss $\mathcal L + 0.5 \cdot \mathcal L_\varepsilon$ (i.e. weighting parameter $\alpha=0.5$) to illustrate robustness to the choice of parameters. }
    \label{fig: simadv stldig}
  \end{center}
  \vspace{-10pt}
\end{figure}

\subsection{Comparing IFM and  ACL(DS)}
\label{appendix: acl comparison}

We give details for Fig \ref{fig: simadv acl comparison}. Similarly to concurrent work \cite{ho2020contrastive,kim2020adversarial}, ACL~\cite{jiang2020robust} directly performs PGD attacks in input space. We compare to the top performing version  ACL(DS) -- which uses a duel stream structure and combines standard and adversarial loss terms. We use the official ACL implementation\footnote{\url{https://github.com/VITA-Group/Adversarial-Contrastive-Learning}} and for fair comparison run IFM by changing only the loss function. All hyperparameters are kept the same for both runs, and follow the ACL recommendations.

\paragraph{Training} We use the SimCLR framework with a ResNet-18 backbone and train for 1000 epochs. We use a base learning rate of 5 with cosine annealing scheduling and batch size 512. LARS optimizer is used. For ACL(DS), we run the PGD for 5 steps in the pre-training stage following the practice of~\cite{jiang2020robust}.
%TODO: full hyperparameter detail's. 
\paragraph{Linear evaluation} We use two schemes to evaluate the quality of learnt representation: standard accuracy and robust accuracy. Robust accuracy reports the accuracy in the setting where an adversary is allowed to apply an $\ell_\infty$ attack to each input. For standard accuracy, we only finetune the last layer and test on clean images following the practice of  MoCo-v2~\cite{chen2020mocov2}.
The initial learning rate is set as 0.1 and we tune for 100 epochs for CIFAR10, 25 epochs for CIFAR100 respectively. An SGD optimizer is used to finetune the model. We use a step scheduler that decreases the learning rate by a factor of $10$ after epochs: $40, 60$ for CIFAR10; $15, 20$ for CIFAR100 respectively.
For robust accuracy, we finetune the model using the loss in TRADE~\cite{zhang2019theoretically}, and evaluate classification accuracy on adversarially
perturbed testing images. We use the same hyperparameters as ACL~\cite{jiang2020robust} for adversarial finetuning.
We perform experiments on CIFAR10 and CIFAR100 and the results are shown in Fig.~\ref{fig: simadv acl comparison}.

\paragraph{Results} See Fig. \ref{fig: simadv acl comparison} in the main manuscript for the results. There are significant qualitative differences between the behaviour of IFM and ACL(DS). IFM improves (standard) linear readout accuracy with zero memory or compute time cost increase, whereas ACL(DS) has improved adversarial linear readout performance, but at the cost of worse standard linear readout and $2\times$  memory and $6\times$ time per epoch. This shows that these two method are addressing two distinct problems. ACL(DS) is suitable for improving the adversarial robustness of a model, whereas IFM improves the generalization of a representation. 

\subsection{Object classification experiments}
\label{appdx: Object classification experimental setup details}

We first describe the protocol used for evaluating IFM on the following datasets: CIFAR10, CIFAR100, STL10, tinyImageNet. For simplicity, the objective weighting parameter is fixed at $\alpha=1$. %Official train/val data splits are used for all datasets.
For MoCo-v2, we performed 5-fold cross validation for CIFAR10/CIFAR100 datasets, and 3 replicated runs on official train/val data splits for tinyImageNet and STL10 datasets.

\paragraph{Training} All encoders have ResNet-50 backbones and are trained for 400 epochs with temperature $\tau=0.5$ for SimCLR and $\tau=0.1$ for MoCo-v2. Encoded features have dimension $2048$ and are followed by a two layer MLP projection head with output dimension $128$. Batch size is taken to be $256$, yielding negative batches of size $m= 510$ for SimCLR. For MoCo-v2, we use a queue size of $k=4096$ (except for STL10 dataset we use $k=8192$), and we use batch size of $256$ for CIFAR10, CIFAR100 and tinyImageNet, $128$ for STL10. For both SimCLR and MoCo-v2 we use the Adam optimizer. 

SimCLR uses initial learning rate $1\times10^{-3}$ and weight decay $1\times10^{-6}$ for CIFAR10, CIFAR100 and tinyImageNet, while STL10 uses $1\times10^{-1}$ learning rate, and weight decay $5\times10^{-4}$ (since we found these settings boosted performance by around $5\%$ in absolute terms). MoCo-v2 training uses weight decay $5\times10^{-4}$, and an initial learning rate $3\times10^{-2}$ for CIFAR10 and CIFAR100; and learning  rate $1\times10^{-1}$ for STL10 and tinyImageNet. Cosine learning rate schedule is used for MoCo-v2. 

\paragraph{Linear evaluation} Evaluation uses test performance of a linear classifier trained ontop of the learned embedding (with embedding model parameters kept fixed) trained for $100 $ epochs. 

For SimCLR, the batch size is set as $512$, and the linear classifier is trained using the Adam optimizer  with learning rate $1\times10^{-3}$ and weight decay $1\times10^{-6}$, and default PyTorch settings for other hyperparameters. For CIFAR10 and CIFAR100 the same augmentations as SimCLR are used for linear classifier training, while for STL10 and tinyImageNet no augmentations were used (since we found this improves performance). 

For MoCo-v2, the batch size is set as $256$. Training uses SGD with initial learning rate set to $30$, momentum is set as $0.9$ and a scheduler that reduces the learning rate by a factor of $10\%$ at epoch 30 and 60. The weight decay is $0$. For CIFAR10 and CIFAR100, we normalize images with mean of $[0.4914, 0.4822, 0.4465]$ and standard deviation of $[0.2023, 0.1994, 0.2010]$. For STL10 and tinyImageNet, we normalize images with mean of $[0.485, 0.456, 0.406]$ and standard deviation of $[0.229, 0.224, 0.225]$.
The same augmentations as the official MoCo-v2 implementation are used for linear classifier training.

%\vspace{-10pt}
\subsubsection{ImageNet100}\label{sec: MoCo-v2 for ImageNet100}
%\vspace{-10pt}
We adopt the official MoCo-v2 code\footnote{\url{https://github.com/facebookresearch/moco}} (CC-BY-NC 4.0 license), modifying only the loss function. For comparison with AdCo method, we adopt the official code\footnote{\url{https://github.com/maple-research-lab/AdCo}} (MIT license) and use the exact same hyperparmeters as for MoCo-v2. For the AdCo specific parameters we perform a simple grid search for the following two hyperparameters: negatives learning rate $lr_\text{neg}$ and negatives temperature $\tau_\text{neg}$. We search over all combinations $lr_\text{neg} \in \{ 1,2,3,4\}$ and $\tau_\text{neg} \in \{ 0.02,0.1\}$, which includes the AdCo default ImageNet1K recommendations $lr_\text{neg} =3$ and $\tau_\text{neg}=0.02$ \cite{hu2020adco}. The result reported for AdCo in Tab. \ref{tab: imagenet100} is the best performance over all $8$ runs. 

\vspace{-10pt}

\paragraph{Training} We use ResNet-50 backbones, and train for 200 epochs. We use a base learning rate of $0.8$ with cosine annealing scheduling and batch size $512$. The MoCo momentum is set to $0.99$, and temperature to $\tau=0.2$. All other hyperparameters are kept the same as the official defaults. 
\vspace{-10pt}
\paragraph{Linear evaluation} We train for $60$ epochs with batch size $128$. We use initial learning rate of $30.0$ and a step scheduler that decreases the learning rate by a factor of $10$ after epochs: $30,40,50$. All other hyperparameters are kept the same as the official MoCo-v2 defaults. 

As noted in the manuscript, our combination of training and linear evaluation parameters leads to $80.5\%$ top-1 linear readout for standard MoCo-v2, and $81.4\%$ with IFM-MoCo-v2. The standard MoCo-v2 performance of $80.5\%$ is, to the best of our knowledge, state-of-the-art performance on ImageNet100 using $200$ epoch training with MoCo-v2. For comparison, we found that using the default recommended MoCo-v2 ImageNet1k parameters (both training and linear evaluation) achieves  ImageNet100 performance of $71.8\%$. This choice of parameters maybe useful for other researchers using MoCo-v2 as a baseline on ImageNet100. 

\subsection{COPDGene dataset}\label{appdx: COPD}

The dataset~\cite{regan2011genetic} in our experiments includes 9,180 subjects. Each subject has a high-resolution inspiratory CT scan and five COPD related outcomes, including two continuous spirometry measures: (1) $\text{FEV1pp}$: the forced expiratory volume in one second, (2) $\text{FEV}_1/\text{FVC}$: the $\text{FEV1pp}$ and forced vital capacity ($\text{FVC}$) ratio, and three ordinal variables: (1) six-grade centrilobular emphysema (CLE) visual score, (2) three-grade paraseptal emphysema (Para-septal) visual score, (3) five-grade dyspnea symptom (mMRC) scale. The dataset is publicly available.

For fair comparison, we use the same encoder and data augmentation described in the baseline approach~\cite{sun2020context}. We set the representation dimension to 128 in all experiments. For simplicity, instead of using a GNN, we use average pooling to aggregate patch representations into image representation. The learning rate is set as $0.01$. We use Adam optimizer and set momentum as 0.9 and weight decay as $1\times10^{-4}$. The batch size is set as 128, and the model is trained for 10 epochs.

\subsection{Further discussion of feature robustness experiments (Sec. \ref{sec: robustnes})}\label{appdx: feature robustness}

Ilyas et al. \cite{ilyas2019adversarial} showed that deep networks richly represent so-called ``non-robust'' features, but that adversarial training can be used to avoid extracting non-robust features at a modest cost to downstream performance. Although in-distribution performance is harmed, Ilyas et al. argue that the reduction in use of non-robust features -- which are highly likely to be statistical coincidences due to the high  dimensionality of input data in computer vision -- may be desirable from the point of view of trustworthiness of a model under input distribution shifts. In this section we consider similar questions on the effect implicit feature modification on  learning of robust vs. non-robust features during self-supervised pre-training. 

Compared to supervised adversarial training \cite{ilyas2019adversarial,madry2017towards} our approach has the key conceptual difference of being applied in feature space. As well as improved computation efficiency (no PGD attacks required) Fig. \ref{fig: robust features} shows that this difference translates into different behavior when using implicit feature modification. Instead of suppressing non-robust features as Ilyas et al. observe for supervised representations, IFM \emph{enhances the representation of robust features}. This suggests that the  improved generalization of encoders trained with IFM can be attributed to improved extraction of features aligned with human semantics (robust features). However, we also note that IFM has no significant effect on learning of non-robust features. In Appdx. \ref{sec: limitations} we discuss the idea of combining IFM with adversarial training methods to get the best of both worlds.

\section{Discussion of limitations and possible extensions}\label{sec: limitations}

While our work makes progress towards understanding, and controlling, feature learning in contrastive self-supervised learning, there still remain many open problems and questions. First, since our proposed implicit feature modification method acts on embedded points instead of raw input data it is not well suited to improving $\ell_p$ robustness, and similarly is not suited to removing pixel-level  shortcut solutions. Instead our method focuses on high-level semantic features. It would be valuable to study the properties of our high-level method used in conjunction with existing pixel-level methods.

Second, while we show that our proposed implicit feature modification method is successful in improving the representation of multiple features simultaneously, our method does not admit an immediate method for determining \emph{which} features are removed during our modification step (i.e. which features are currently being used to solve the instance discrimination task). One option is to manually study examples using the visualization technique we propose in Sec. \ref{sec: vis}.

%% file: main.bbl
\begin{thebibliography}{56}
\providecommand{\natexlab}[1]{#1}
\providecommand{\url}[1]{\texttt{#1}}
\expandafter\ifx\csname urlstyle\endcsname\relax
  \providecommand{\doi}[1]{doi: #1}\else
  \providecommand{\doi}{doi: \begingroup \urlstyle{rm}\Url}\fi

\bibitem[Bardes et~al.(2021)Bardes, Ponce, and LeCun]{bardes2021VICReg}
Adrien Bardes, Jean Ponce, and Yann LeCun.
\newblock {VICR}eg: Variance-invariance-covariance regularization for
  self-supervised learning.
\newblock \emph{preprint arXiv:2105.04906}, 2021.

\bibitem[Beery et~al.(2018)Beery, Van~Horn, and Perona]{beery2018recognition}
Sara Beery, Grant Van~Horn, and Pietro Perona.
\newblock Recognition in terra incognita.
\newblock In \emph{Proceedings of the European Conference on Computer Vision
  (ECCV)}, pages 456--473, 2018.

\bibitem[Caron et~al.(2020)Caron, Misra, Mairal, Goyal, Bojanowski, and
  Joulin]{caron2020unsupervised}
Mathilde Caron, Ishan Misra, Julien Mairal, Priya Goyal, Piotr Bojanowski, and
  Armand Joulin.
\newblock Unsupervised learning of visual features by contrasting cluster
  assignments.
\newblock In \emph{Advances in Neural Information Processing Systems
  (NeurIPS)}, pages 9912--9924, 2020.

\bibitem[Chen and Li(2020)]{chen2020intriguing}
Ting Chen and Lala Li.
\newblock Intriguing properties of contrastive losses.
\newblock \emph{preprint arXiv:2011.02803}, 2020.

\bibitem[Chen et~al.(2020{\natexlab{a}})Chen, Kornblith, Norouzi, and
  Hinton]{chen2020simple}
Ting Chen, Simon Kornblith, Mohammad Norouzi, and Geoffrey Hinton.
\newblock A simple framework for contrastive learning of visual
  representations.
\newblock In \emph{Int. Conference on Machine Learning (ICML)}, pages
  10709--10719, 2020{\natexlab{a}}.

\bibitem[Chen and He(2021)]{chen2020exploring}
Xinlei Chen and Kaiming He.
\newblock Exploring simple siamese representation learning.
\newblock In \emph{IEEE Conference on Computer Vision and Pattern Recognition
  (CVPR)}, 2021.

\bibitem[Chen et~al.(2020{\natexlab{b}})Chen, Fan, Girshick, and
  He]{chen2020mocov2}
Xinlei Chen, Haoqi Fan, Ross Girshick, and Kaiming He.
\newblock Improved baselines with momentum contrastive learning.
\newblock \emph{preprint arXiv:2003.04297}, 2020{\natexlab{b}}.

\bibitem[Chizat and Bach(2020)]{chizat2020implicit}
Lenaic Chizat and Francis Bach.
\newblock Implicit bias of gradient descent for wide two-layer neural networks
  trained with the logistic loss.
\newblock In \emph{Conference on Learning Theory (COLT)}, pages 1305--1338,
  2020.

\bibitem[Chuang et~al.(2020)Chuang, Robinson, Yen-Chen, Torralba, and
  Jegelka]{chuang2020debiased}
Ching-Yao Chuang, Joshua Robinson, Lin Yen-Chen, Antonio Torralba, and Stefanie
  Jegelka.
\newblock Debiased contrastive learning.
\newblock In \emph{Advances in Neural Information Processing Systems
  (NeurIPS)}, pages 8765--8775, 2020.

\bibitem[Geirhos et~al.(2019)Geirhos, Rubisch, Michaelis, Bethge, Wichmann, and
  Brendel]{geirhos2018imagenet}
Robert Geirhos, Patricia Rubisch, Claudio Michaelis, Matthias Bethge, Felix~A
  Wichmann, and Wieland Brendel.
\newblock Image{N}et-trained {CNN}s are biased towards texture; increasing
  shape bias improves accuracy and robustness.
\newblock In \emph{Int. Conf. on Learning Representations (ICLR)}, 2019.

\bibitem[Geirhos et~al.(2020)Geirhos, Jacobsen, Michaelis, Zemel, Brendel,
  Bethge, and Wichmann]{geirhos2020shortcut}
Robert Geirhos, J{\"o}rn-Henrik Jacobsen, Claudio Michaelis, Richard Zemel,
  Wieland Brendel, Matthias Bethge, and Felix~A Wichmann.
\newblock Shortcut learning in deep neural networks.
\newblock \emph{Nature Machine Intelligence}, 2\penalty0 (11):\penalty0
  665--673, 2020.

\bibitem[Goodfellow et~al.(2015)Goodfellow, Shlens, and
  Szegedy]{goodfellow2014explaining}
Ian~J Goodfellow, Jonathon Shlens, and Christian Szegedy.
\newblock Explaining and harnessing adversarial examples.
\newblock In \emph{Int. Conf. on Learning Representations (ICLR)}, 2015.

\bibitem[Grill et~al.(2020)Grill, Strub, Altch{\'e}, Tallec, Richemond,
  Buchatskaya, Doersch, Pires, Guo, Azar, et~al.]{grill2020bootstrap}
Jean-Bastien Grill, Florian Strub, Florent Altch{\'e}, Corentin Tallec,
  Pierre~H Richemond, Elena Buchatskaya, Carl Doersch, Bernardo~Avila Pires,
  Zhaohan~Daniel Guo, Mohammad~Gheshlaghi Azar, et~al.
\newblock Bootstrap your own latent: A new approach to self-supervised
  learning.
\newblock In \emph{Advances in Neural Information Processing Systems
  (NeurIPS)}, pages 21271--21284, 2020.

\bibitem[Gutmann and Hyv{\"a}rinen(2010)]{gutmann2010noise}
Michael Gutmann and Aapo Hyv{\"a}rinen.
\newblock Noise-contrastive estimation: A new estimation principle for
  unnormalized statistical models.
\newblock In \emph{Proc. Int. Conference on Artificial Intelligence and
  Statistics (AISTATS)}, pages 297--304, 2010.

\bibitem[He et~al.(2020)He, Fan, Wu, Xie, and Girshick]{he2020momentum}
Kaiming He, Haoqi Fan, Yuxin Wu, Saining Xie, and Ross Girshick.
\newblock Momentum contrast for unsupervised visual representation learning.
\newblock In \emph{IEEE Conference on Computer Vision and Pattern Recognition
  (CVPR)}, pages 9729--9738, 2020.

\bibitem[Hermann and Lampinen(2020)]{hermann2020shapes}
Katherine~L Hermann and Andrew~K Lampinen.
\newblock What shapes feature representations? {E}xploring datasets,
  architectures, and training.
\newblock In \emph{Advances in Neural Information Processing Systems
  (NeurIPS)}, pages 9995--10006, 2020.

\bibitem[Hermann et~al.(2019)Hermann, Chen, and Kornblith]{hermann2019origins}
Katherine~L Hermann, Ting Chen, and Simon Kornblith.
\newblock The origins and prevalence of texture bias in convolutional neural
  networks.
\newblock In \emph{Advances in Neural Information Processing Systems
  (NeurIPS)}, pages 19000--19015, 2019.

\bibitem[Hjelm et~al.(2019)Hjelm, Fedorov, Lavoie-Marchildon, Grewal, Bachman,
  Trischler, and Bengio]{hjelm2018learning}
R~Devon Hjelm, Alex Fedorov, Samuel Lavoie-Marchildon, Karan Grewal, Phil
  Bachman, Adam Trischler, and Yoshua Bengio.
\newblock Learning deep representations by mutual information estimation and
  maximization.
\newblock In \emph{Int. Conf. on Learning Representations (ICLR)}, 2019.

\bibitem[Ho and Nvasconcelos(2020)]{ho2020contrastive}
Chih-Hui Ho and Nuno Nvasconcelos.
\newblock Contrastive learning with adversarial examples.
\newblock In \emph{Advances in Neural Information Processing Systems
  (NeurIPS)}, pages 17081--17093, 2020.

\bibitem[Hu et~al.(2020)Hu, Wang, Hu, and Qi]{hu2020adco}
Qianjiang Hu, Xiao Wang, Wei Hu, and Guo-Jun Qi.
\newblock Adco: Adversarial contrast for efficient learning of unsupervised
  representations from self-trained negative adversaries.
\newblock In \emph{IEEE Conference on Computer Vision and Pattern Recognition
  (CVPR)}, 2020.

\bibitem[Huh et~al.(2021)Huh, Mobahi, Zhang, Cheung, Agrawal, and
  Isola]{huh2021low}
Minyoung Huh, Hossein Mobahi, Richard Zhang, Brian Cheung, Pulkit Agrawal, and
  Phillip Isola.
\newblock The low-rank simplicity bias in deep networks.
\newblock \emph{preprint arXiv:2103.10427}, 2021.

\bibitem[Ilyas et~al.(2019)Ilyas, Santurkar, Tsipras, Engstrom, Tran, and
  Madry]{ilyas2019adversarial}
Andrew Ilyas, Shibani Santurkar, Dimitris Tsipras, Logan Engstrom, Brandon
  Tran, and Aleksander Madry.
\newblock Adversarial examples are not bugs, they are features.
\newblock In \emph{Advances in Neural Information Processing Systems
  (NeurIPS)}, pages 125--136, 2019.

\bibitem[Jacobsen et~al.(2018)Jacobsen, Behrmann, Zemel, and
  Bethge]{jacobsen2018excessive}
J{\"o}rn-Henrik Jacobsen, Jens Behrmann, Richard Zemel, and Matthias Bethge.
\newblock Excessive invariance causes adversarial vulnerability.
\newblock In \emph{Int. Conf. on Learning Representations (ICLR)}, 2018.

\bibitem[Jiang et~al.(2020)Jiang, Chen, Chen, and Wang]{jiang2020robust}
Ziyu Jiang, Tianlong Chen, Ting Chen, and Zhangyang Wang.
\newblock Robust pre-training by adversarial contrastive learning.
\newblock In \emph{Advances in Neural Information Processing Systems
  (NeurIPS)}, pages 16199--16210, 2020.

\bibitem[Kalantidis et~al.(2020)Kalantidis, Sariyildiz, Pion, Weinzaepfel, and
  Larlus]{kalantidis2020hard}
Yannis Kalantidis, Mert~Bulent Sariyildiz, Noe Pion, Philippe Weinzaepfel, and
  Diane Larlus.
\newblock Hard negative mixing for contrastive learning.
\newblock In \emph{Advances in Neural Information Processing Systems
  (NeurIPS)}, pages 21798--21809, 2020.

\bibitem[Kim et~al.(2020)Kim, Tack, and Hwang]{kim2020adversarial}
Minseon Kim, Jihoon Tack, and Sung~Ju Hwang.
\newblock Adversarial self-supervised contrastive learning.
\newblock In \emph{Advances in Neural Information Processing Systems
  (NeurIPS)}, 2020.

\bibitem[Lee et~al.(2020)Lee, Lei, Saunshi, and Zhuo]{lee2020predicting}
Jason~D Lee, Qi~Lei, Nikunj Saunshi, and Jiacheng Zhuo.
\newblock Predicting what you already know helps: Provable self-supervised
  learning.
\newblock \emph{preprint arXiv:2008.01064}, 2020.

\bibitem[Li et~al.(2020)Li, Fan, Yuan, He, Tian, and Katabi]{li2020information}
Tianhong Li, Lijie Fan, Yuan Yuan, Hao He, Yonglong Tian, and Dina Katabi.
\newblock Information-preserving contrastive learning for self-supervised
  representations.
\newblock \emph{preprint arXiv:2012.09962}, 2020.

\bibitem[Lyu and Li(2020)]{lyu2019gradient}
Kaifeng Lyu and Jian Li.
\newblock Gradient descent maximizes the margin of homogeneous neural networks.
\newblock In \emph{Int. Conf. on Learning Representations (ICLR)}, 2020.

\bibitem[Madry et~al.(2018)Madry, Makelov, Schmidt, Tsipras, and
  Vladu]{madry2017towards}
Aleksander Madry, Aleksandar Makelov, Ludwig Schmidt, Dimitris Tsipras, and
  Adrian Vladu.
\newblock Towards deep learning models resistant to adversarial attacks.
\newblock In \emph{Int. Conf. on Learning Representations (ICLR)}, 2018.

\bibitem[Minderer et~al.(2020)Minderer, Bachem, Houlsby, and
  Tschannen]{minderer2020automatic}
Matthias Minderer, Olivier Bachem, Neil Houlsby, and Michael Tschannen.
\newblock Automatic shortcut removal for self-supervised representation
  learning.
\newblock In \emph{International Conference on Machine Learning}, pages
  6927--6937, 2020.

\bibitem[Nguyen et~al.(2021)Nguyen, Raghu, and Kornblith]{nguyen2020wide}
Thao Nguyen, Maithra Raghu, and Simon Kornblith.
\newblock Do wide and deep networks learn the same things? uncovering how
  neural network representations vary with width and depth.
\newblock In \emph{Int. Conf. on Learning Representations (ICLR)}, 2021.

\bibitem[Oord et~al.(2018)Oord, Li, and Vinyals]{oord2018representation}
Aaron van~den Oord, Yazhe Li, and Oriol Vinyals.
\newblock Representation learning with contrastive predictive coding.
\newblock \emph{preprint arXiv:1807.03748}, 2018.

\bibitem[Paszke et~al.(2019)Paszke, Gross, Massa, Lerer, Bradbury, Chanan,
  Killeen, Lin, Gimelshein, Antiga, et~al.]{paszke2019pytorch}
Adam Paszke, Sam Gross, Francisco Massa, Adam Lerer, James Bradbury, Gregory
  Chanan, Trevor Killeen, Zeming Lin, Natalia Gimelshein, Luca Antiga, et~al.
\newblock Pytorch: An imperative style, high-performance deep learning library.
\newblock In \emph{Advances in Neural Information Processing Systems
  (NeurIPS)}, 2019.

\bibitem[Pathak et~al.(2016)Pathak, Krahenbuhl, Donahue, Darrell, and
  Efros]{pathak2016context}
Deepak Pathak, Philipp Krahenbuhl, Jeff Donahue, Trevor Darrell, and Alexei~A
  Efros.
\newblock Context encoders: Feature learning by inpainting.
\newblock In \emph{IEEE Conference on Computer Vision and Pattern Recognition
  (CVPR)}, pages 2536--2544, 2016.

\bibitem[Pedregosa et~al.(2011)Pedregosa, Varoquaux, Gramfort, Michel, Thirion,
  Grisel, Blondel, Prettenhofer, Weiss, Dubourg, et~al.]{pedregosa2011scikit}
Fabian Pedregosa, Ga{\"e}l Varoquaux, Alexandre Gramfort, Vincent Michel,
  Bertrand Thirion, Olivier Grisel, Mathieu Blondel, Peter Prettenhofer, Ron
  Weiss, Vincent Dubourg, et~al.
\newblock Scikit-learn: Machine learning in python.
\newblock \emph{Journal of machine learning research}, pages 2825--2830, 2011.

\bibitem[Recht et~al.(2019)Recht, Roelofs, Schmidt, and
  Shankar]{recht2019imagenet}
Benjamin Recht, Rebecca Roelofs, Ludwig Schmidt, and Vaishaal Shankar.
\newblock Do {I}mage{N}et classifiers generalize to {I}mage{N}et?
\newblock In \emph{Int. Conference on Machine Learning (ICML)}, pages
  5389--5400, 2019.

\bibitem[Regan et~al.(2011)Regan, Hokanson, Murphy, Make, Lynch, Beaty,
  Curran-Everett, Silverman, and Crapo]{regan2011genetic}
Elizabeth~A Regan, John~E Hokanson, James~R Murphy, Barry Make, David~A Lynch,
  Terri~H Beaty, Douglas Curran-Everett, Edwin~K Silverman, and James~D Crapo.
\newblock Genetic epidemiology of {COPD} ({COPDGene}) study design.
\newblock \emph{{COPD}: Journal of Chronic Obstructive Pulmonary Disease},
  7\penalty0 (1):\penalty0 32--43, 2011.

\bibitem[Robinson et~al.(2020)Robinson, Jegelka, and Sra]{robinson2020strength}
Joshua Robinson, Stefanie Jegelka, and Suvrit Sra.
\newblock Strength from weakness: Fast learning using weak supervision.
\newblock In \emph{Int. Conference on Machine Learning (ICML)}, pages
  8127--8136, 2020.

\bibitem[Robinson et~al.(2021)Robinson, Chuang, Sra, and
  Jegelka]{robinson2020contrastive}
Joshua Robinson, Ching-Yao Chuang, Suvrit Sra, and Stefanie Jegelka.
\newblock Contrastive learning with hard negative samples.
\newblock In \emph{Int. Conf. on Learning Representations (ICLR)}, 2021.

\bibitem[Soudry et~al.(2018)Soudry, Hoffer, Nacson, Gunasekar, and
  Srebro]{soudry2018implicit}
Daniel Soudry, Elad Hoffer, Mor~Shpigel Nacson, Suriya Gunasekar, and Nathan
  Srebro.
\newblock The implicit bias of gradient descent on separable data.
\newblock \emph{The Journal of Machine Learning Research}, 19\penalty0
  (1):\penalty0 2822--2878, 2018.

\bibitem[Sun et~al.(2021)Sun, Yu, and Batmanghelich]{sun2020context}
Li~Sun, Ke~Yu, and Kayhan Batmanghelich.
\newblock Context matters: Graph-based self-supervised representation learning
  for medical images.
\newblock In \emph{Proceedings of the AAAI Conference on Artificial
  Intelligence}, volume~35, pages 4874--4882, 2021.

\bibitem[Szegedy et~al.(2014)Szegedy, Zaremba, Sutskever, Bruna, Erhan,
  Goodfellow, and Fergus]{szegedy2013intriguing}
Christian Szegedy, Wojciech Zaremba, Ilya Sutskever, Joan Bruna, Dumitru Erhan,
  Ian Goodfellow, and Rob Fergus.
\newblock Intriguing properties of neural networks.
\newblock In \emph{Int. Conf. on Learning Representations (ICLR)}, 2014.

\bibitem[Tian et~al.(2020{\natexlab{a}})Tian, Krishnan, and
  Isola]{tian2019contrastive}
Yonglong Tian, Dilip Krishnan, and Phillip Isola.
\newblock Contrastive multiview coding.
\newblock In \emph{Europ. Conference on Computer Vision (ECCV)},
  2020{\natexlab{a}}.

\bibitem[Tian et~al.(2020{\natexlab{b}})Tian, Sun, Poole, Krishnan, Schmid, and
  Isola]{tian2020makes}
Yonglong Tian, Chen Sun, Ben Poole, Dilip Krishnan, Cordelia Schmid, and
  Phillip Isola.
\newblock What makes for good views for contrastive learning.
\newblock In \emph{Advances in Neural Information Processing Systems
  (NeurIPS)}, pages 6827--6839, 2020{\natexlab{b}}.

\bibitem[Tsipras et~al.(2018)Tsipras, Santurkar, Engstrom, Turner, and
  Madry]{tsipras2018robustness}
Dimitris Tsipras, Shibani Santurkar, Logan Engstrom, Alexander Turner, and
  Aleksander Madry.
\newblock Robustness may be at odds with accuracy.
\newblock In \emph{Int. Conf. on Learning Representations (ICLR)}, 2018.

\bibitem[Wang and Liu(2021)]{wang2020understandingbehaviour}
Feng Wang and Huaping Liu.
\newblock Understanding the behaviour of contrastive loss.
\newblock In \emph{IEEE Conference on Computer Vision and Pattern Recognition
  (CVPR)}, 2021.

\bibitem[Wang et~al.(2020)Wang, Liu, Guo, and Sun]{wang2020unsupervised}
Feng Wang, Huaping Liu, Di~Guo, and Fuchun Sun.
\newblock Unsupervised representation learning by invariance propagation.
\newblock In \emph{Advances in Neural Information Processing Systems
  (NeurIPS)}, pages 3510--3520, 2020.

\bibitem[Wang and Isola(2020)]{wang2020understanding}
Tongzhou Wang and Phillip Isola.
\newblock Understanding contrastive representation learning through alignment
  and uniformity on the hypersphere.
\newblock In \emph{Int. Conference on Machine Learning (ICML)}, pages
  9574--9584, 2020.

\bibitem[Wang and Qi(2021)]{wang2021augmentations}
Xiao Wang and Guo-Jun Qi.
\newblock Contrastive learning with stronger augmentations.
\newblock \emph{IEEE Transactions on Pattern Analysis and Machine
  Intelligence}, 2021.

\bibitem[Zbontar et~al.(2021)Zbontar, Jing, Misra, LeCun, and
  Deny]{zbontar2021barlow}
Jure Zbontar, Li~Jing, Ishan Misra, Yann LeCun, and St{\'e}phane Deny.
\newblock Barlow twins: Self-supervised learning via redundancy reduction.
\newblock In \emph{Int. Conference on Machine Learning (ICML)}, 2021.

\bibitem[Zhang et~al.(2019)Zhang, Yu, Jiao, Xing, El~Ghaoui, and
  Jordan]{zhang2019theoretically}
Hongyang Zhang, Yaodong Yu, Jiantao Jiao, Eric Xing, Laurent El~Ghaoui, and
  Michael Jordan.
\newblock Theoretically principled trade-off between robustness and accuracy.
\newblock In \emph{Int. Conference on Machine Learning (ICML)}, pages
  7472--7482, 2019.

\bibitem[Zhang et~al.(2018)Zhang, Cisse, Dauphin, and
  Lopez-Paz]{zhang2017mixup}
Hongyi Zhang, Moustapha Cisse, Yann~N Dauphin, and David Lopez-Paz.
\newblock mixup: Beyond empirical risk minimization.
\newblock In \emph{Int. Conf. on Learning Representations (ICLR)}, 2018.

\bibitem[Zhang et~al.(2016)Zhang, Isola, and Efros]{zhang2016colorful}
Richard Zhang, Phillip Isola, and Alexei~A Efros.
\newblock Colorful image colorization.
\newblock In \emph{Europ. Conference on Computer Vision (ECCV)}, pages
  649--666, 2016.

\bibitem[Zhao et~al.(2021)Zhao, Wu, Lau, and Lin]{zhao2020makes}
Nanxuan Zhao, Zhirong Wu, Rynson~WH Lau, and Stephen Lin.
\newblock What makes instance discrimination good for transfer learning?
\newblock In \emph{Int. Conf. on Learning Representations (ICLR)}, 2021.

\bibitem[Zimmermann et~al.(2021)Zimmermann, Sharma, Schneider, Bethge, and
  Brendel]{zimmermann2021contrastive}
Roland~S Zimmermann, Yash Sharma, Steffen Schneider, Matthias Bethge, and
  Wieland Brendel.
\newblock Contrastive learning inverts the data generating process.
\newblock In \emph{Int. Conference on Machine Learning (ICML)}, 2021.

\end{thebibliography}
